\documentclass[10pt]{article}

\usepackage{microtype}
\usepackage[utf8]{inputenc}
\usepackage{amsmath,amssymb,amsfonts}
\usepackage{amsthm}
\usepackage{thmtools,thm-restate}
\usepackage{bm}
\usepackage{mathtools}
\usepackage{graphicx}
\usepackage{booktabs}

\usepackage{multirow}
\usepackage{tabularx}
\usepackage{subcaption}
\usepackage{array}
\usepackage[dvipsnames]{xcolor}

\usepackage{hyperref}
\usepackage{amdstyle}
\usepackage{etoc}

\usepackage{natbib}
\setcitestyle{authoryear,round,citesep={;},aysep={,},yysep={;}}
\renewcommand{\cite}[1]{\citep{#1}}
\newcommand{\yrcite}[1]{\citeyearpar{#1}}

\theoremstyle{definition}
\newtheorem{definition}{Definition}[section]

\theoremstyle{remark}
\newtheorem{remark}{Remark}[section]

\makeatletter
\let\c@lemma\c@theorem
\let\c@corollary\c@theorem
\let\c@proposition\c@theorem
\let\c@definition\c@theorem
\let\c@remark\c@theorem
\makeatother

\newcommand{\clip}{\textrm{clip}}
\DeclareMathOperator*{\argmin}{arg\,min}
\DeclareMathOperator*{\argmax}{arg\,max}

\newcommand{\ignore}[1]{}

\newcommand{\papertitle}{Signed Symmetric Quantization for Few-Bit Integers}
\newcommand{\paperauthors}{%
  Ian Colbert\textsuperscript{*}, Eashan Dash\textsuperscript{*}, Pablo Monteagudo-Lago, Juan Amboage, \\
  Srinidhi N, Giuseppe Franco, Nicholas J. Fraser, Arun Ramachandran
}

\begin{document}

\etocsettocdepth.toc{none}

\amdsetrunningtitle{\papertitle}
\amdmaketitle{\papertitle}{\paperauthors}

{\renewcommand{\thefootnote}{\fnsymbol{footnote}}
\footnotetext[1]{Correspondence to: \texttt{ian.colbert@amd.com}, \texttt{eashan.dash@amd.com}}}

\begin{abstract}
The signed integer alphabet contains one more negative representable value than positive. Yet, by convention, the standard symmetric integer quantizer takes its scale to be strictly positive, fixing the extra representable value on the negative tail and possibly clipping positive outliers.
In this work, we show that such clipping is a non-trivial source of quantization error at few-bit precision.
Asymmetric quantization can address this problem with a zero point, shifting the grid toward the observed data range; however, this flexibility is well-known to carry inference overhead.
We therefore highlight signed symmetric quantization as a third option that retains the runtime profile of symmetric quantization without the overhead of the asymmetric format.
In particular, the sign of the scale chooses which tail receives the extra representable value, toggling the polarity of the grid rather than shifting it with a zero point.
The strictly positive convention is near-universal, but the choice of grid polarity has not yet been systematically analyzed; our theoretical analysis therefore offers three main results.
First, we establish that placing the extra representable value on the dominant-outlier tail is conditionally bound-optimal on $\ell_2$ weight reconstruction error, and show that the condition holds for $88$--$99\%$ of weight groups across the pre-trained large language models (LLMs) we evaluated at low bit widths.
Second, we derive a new closed-form sign rule that directly minimizes the exact $\ell_2$ weight reconstruction error at any fixed scale magnitude.
Third, we show that negating the scale of a standard symmetric quantizer is analytically equivalent to a unit zero point shift on the same signed integer alphabet.
We empirically validate signed symmetric quantization on models from the Qwen3, Qwen3.5, and Llama3 families, and observe improvement in perplexity and downstream reasoning accuracy over the strictly positive absmax grid at no extra inference cost.
\end{abstract}

\section{Introduction}
\label{sec:intro}
Few-bit quantization relies heavily on uniform integer grids uniquely defined by an alphabet, a scale factor, and a zero point. A small but often overlooked detail is that the standard signed $q$-bit alphabet $\mathcal{A}_q = \{-2^{q-1}, \ldots, 2^{q-1}-1\}$ is not symmetric: it contains one more negative value than positive. 
Yet, by convention, the standard symmetric quantizer takes the scale $s$ to be positive~\cite{krishnamoorthi2018quantizing, gholami2022survey}, fixing the extra representable value on the negative tail of the data and possibly clipping large positive values. We show in this work that this can become a material source of quantization error at few-bit precision.

Asymmetric quantization can address this clipping by introducing a zero point to align the grid with the data range, but this flexibility is well-known to carry an inference cost~\cite{jacob2018quantization, krishnamoorthi2018quantizing, gholami2022survey}. Kernels must account for the offset, and the corresponding metadata must be stored and loaded. For example, on Llama3 8B, we find that the symmetric 4-bit format \texttt{Q4\_0} uses 9\% less memory and provides $2.21\times$ higher prefill and $2.03\times$ higher decode throughput than its asymmetric counterpart \texttt{Q4\_1} as measured using the \texttt{llama.cpp} inference engine~\cite{llama_cpp} on an AMD EPYC\texttrademark{} ``Turin'' CPU (Table~\ref{tbl:cost}). This motivates a narrower question: can we approach the flexibility of an asymmetric grid while keeping the inference efficiency of a symmetric one?

We answer this question affirmatively with signed symmetric quantization, where the sign of the scale selects which tail of the data receives the extra representable value.
Choosing that sign costs nothing at inference, as the stored fields and the kernel path are identical to those of the symmetric format.
We study two sign selection rules: one that aligns the dominant outlier in the data with the extra negative representable value in the signed integer alphabet, and one that provably minimizes the exact mean squared error (MSE) at any fixed scale magnitude.
Both match or exceed the asymmetric grid at 2 bits in our experiments, and their gains over the strictly positive convention grow as precision falls.
For example, on Llama3 8B at 2 bits, where the standard symmetric quantizer recovers $65.4\%$ downstream accuracy, the MSE-optimal sign rule recovers $75.1\%$, above even the asymmetric format's $70.8\%$ (Table~\ref{tbl:llama8b_full}).

\paragraph{Contributions.}
We provide the first theoretical formalization of grid polarity and show that it can be exploited without modifying existing runtimes.
\textbf{First,} we formalize signed symmetric quantization as the family of integer quantizers that lift the strictly positive scale constraint, so that the sign of the scale toggles the polarity of the grid (Section~\ref{sec:signed_sym}). 
\textbf{Second,} we prove a worst-case $\ell_2$ error bound (Theorem~\ref{thm:worst-case-bound}) and show that placing the extra representable value on the dominant-outlier tail minimizes it under a mild cardinality condition (Corollary~\ref{cor:optimal-sign}). 
\textbf{Third,} we derive a new closed-form sign rule that directly minimizes the exact $\ell_2$ error at any fixed scale magnitude (Corollary~\ref{cor:mse-optimal}). 
\textbf{Fourth,} we prove that a scale sign flip is analytically equivalent to a unit zero-point shift on the same signed alphabet (Theorem~\ref{thm:signed-asymmetric}). 
\textbf{Fifth,} we evaluate models in the Qwen3, Qwen3.5, and Llama3 families, showing that signed symmetric quantization improves over the strictly positive convention, especially as precision is reduced, and matches or exceeds asymmetric quantization at 2 bits while preserving the symmetric deployment profile.

\paragraph{Notation.} Let $w \in \mathbb{R}^{d}$ denote a $d$-dimensional weight vector, with elements \( w_j \in \mathbb{R} \) for \( j \in [d] \), where \( [d] = \{1, \ldots, d \} \).
A uniform integer quantizer {\( \mathcal{Q} \)} maps real values to a discrete \( q \)-bit alphabet \( \mathcal{A}_q = \{a_{\min}, \ldots, a_{\max}\} \subset \mathbb{Z} \) of consecutive integers, where \( \vert \mathcal{A}_q \vert = 2^q \).
In particular, the standard integer quantizer is defined as
\begin{equation}
\mathcal{Q}(w) = s \cdot \left( \clip\Big(\left\lfloor \frac{w}{s} + z \right\rceil; \mathcal{A}_q \Big) - z \right),
\label{eq:quantizer}
\end{equation}
where \( s \in \mathbb{R} \) is the non-zero scaling factor, \( z \in \mathbb{R} \) is the zero point, \( \lfloor \cdot \rceil \) is the round-to-nearest (RTN) operator, and \( \clip(x;\mathcal{A}_q) := \max(\min(x, a_{\max}), a_{\min}) \) clamps \( x \) to the range of \( \mathcal{A}_q \).
We define the scale in terms of its magnitude and sign: \( s = \gamma\alpha \) with \( \alpha > 0 \) and \( \gamma \in \{-1, +1\} \).
Different quantization schemes (\textit{e.g.}, asymmetric or symmetric) are distinguished by how \( s,  z, \) and \( \mathcal{A}_q \) are chosen, and we refer to unique selections of \( s,  z, \) and \( \mathcal{A}_q \) as a grid.
\section{Background and Related Work}
\label{sec:related-work}

Quantization seeks to map a full-precision tensor onto a discrete, low-precision grid that is uniquely defined by a scale, a zero point, and an alphabet.
Prior work aims to reduce the impact of the resulting error on model accuracy and can be distinguished along two complementary axes: (i) for a fixed grid, adjust the weights to minimize the error; (ii) for fixed weights, choose the grid that minimizes the error.
The two are not mutually exclusive, as we demonstrate in Section~\ref{sec:results}.

\paragraph{Optimizing the weights for a fixed grid.}
Work along this axis concerns one or both aspects of the quantization pipeline.
The first aspect involves transformations that make weights and activations more amenable to quantization; for example, Xiao et al.~\yrcite{xiao2023smoothquant} rescale channels to migrate quantization difficulty from activations to weights, Ashkboos et al.~\yrcite{ashkboos2024quarot} apply Hadamard rotations to suppress outliers, and Liu et al.~\yrcite{liu2025spinquant} extend this with learned rotations.
The second aspect involves rounding schemes that map the (potentially transformed) weights and activations onto a grid; for example, Frantar et al.~\yrcite{frantar2022gptq} introduce GPTQ, a greedy algorithm that uses approximate second-order information from calibration data to correct local quantization error, and Zhang et al.~\yrcite{zhang2026qronos} generalize this with Qronos to account for error propagated from previously quantized layers.
Neither aspect changes the grid directly, which is the focus of our study, but we show that signed symmetric quantization composes with both (Section~\ref{sec:results}).

\paragraph{Optimizing the grid for fixed weights.}
Work along this axis selects the target alphabet as well as the scale and, optionally, the zero point.
Our work focuses on the standard integer alphabet; however, alternatives include non-uniform codebooks tuned to better fit the data.
This added flexibility can yield more accurate models at the same bit width, but can also carry a non-trivial inference cost.
For example, within the \texttt{llama.cpp} framework, the non-uniform \texttt{IQ4\_NL} format yields lower WikiText2 perplexity than its uniform counterpart \texttt{Q4\_0} for Llama2 7B ($4.99$ versus $5.03$) and Llama2 13B ($4.46$ versus $4.47$), while \texttt{Q4\_0} yields up to $1.08\times$ higher prefill throughput and up to $1.21\times$ higher decode throughput across four backends and three hardware vendors ($1.05\times$ and $1.06\times$ on average, respectively)~\cite{kawrakow2024iq4nl}.
We observe this on an AMD EPYC\texttrademark{} ``Turin'' CPU, where \texttt{Q4\_0} delivers $1.04$--$1.08\times$ the prefill and $0.95$--$1.07\times$ the decode throughput of \texttt{IQ4\_NL} (Table~\ref{tbl:cost_iq4nl}, Appendix~\ref{sec:impl-runtime}).
Thus, uniform and non-uniform approaches occupy different points on the frontier of inference optimization; our work focuses on the former, where the grid is set by the scale and zero point using the standard integer alphabet.

Within this setting, Esser et al.~\yrcite{esser2019learned} learn scaling factors via gradient descent, later extended by Bhalgat et al.~\yrcite{bhalgat2020lsq+} to learnable zero points. 
In LLM quantization, however, the scale is typically calculated using heuristics based on data range~\cite{krishnamoorthi2018quantizing, gholami2022survey, zhang2026qronos}, or by greedily minimizing local quantization error~\cite{frantar2022gptq}. 
Recent works instead optimize the scale to minimize local output reconstruction error without explicitly constraining its sign, yet in practice restrict evaluation to positive solutions~\cite{amboage2026piso, zhang2026beacon}. 
In any case, the scale is assumed (either implicitly or explicitly) to be strictly positive, assigning the extra representable value to the negative tail. This convention is widespread but not a theoretical requirement.
Whether that value is used at all is a separate choice, and quantization libraries expose one or both options: the full representation range of the format, or a restriction to a symmetric subset (\textit{i.e.}, narrow range)~\cite{brevitas}.
That choice determines only \emph{whether} the value is used. Taking the full range, we instead ask \emph{which tail} should receive it.
In particular, we study grid polarity as a mechanism to exploit the native asymmetry in the standard integer alphabet without adding complexity to the runtime.

The methods closest to ours select among multiple grids rather than orienting a single one.
AFPQ~\cite{zhang2024afpq} uses the sign of each weight to select between two non-uniform sub-grids, while SA-ANT~\cite{geng2026sa} constructs distinct positive and negative half-grids per group and requires a custom decoder.
BitMoD~\cite{chen2025bitmod} introduces custom FP3 and FP4 alphabets by repurposing negative zero as a per-group special value, chosen to minimize reconstruction error.
OliVe~\cite{guo2023olive} removes $-8$ from the standard 4-bit integer alphabet and repurposes its bit pattern within a custom outlier encoding scheme, leaving a symmetric range of $[-7, 7]$.
Each therefore carries an inference cost, whether from encoding a second grid or requiring a custom decoder.
In contrast, we treat grid polarity itself as the object of analysis, choosing between a grid and its reflection with the sign of the scale, once per group; this introduces no second scale or alphabet, no per-element routing, and no custom decode path.
The target alphabet is unchanged, every integer code retains its standard numerical meaning, and the format's native asymmetry is matched to the asymmetry that emerges in the data.
To our knowledge, no prior work studies the polarity of a fixed integer grid through the sign of its scale; our contribution is therefore the first theoretical analysis of this degree of freedom, together with a new closed-form rule for exploiting it within existing runtimes such as \texttt{llama.cpp} without modification.

\section{Signed Symmetric Quantization}
\label{sec:signed_sym}

Signed symmetric quantization exploits the native asymmetry of the signed $q$-bit alphabet. The strictly positive convention always assigns the extra representable value to the negative tail, which is often negligible at higher precision but can dominate quantization error as the bit width $q$ is reduced, as later made explicit in Theorem~\ref{thm:worst-case-bound}.

Asymmetric quantization shifts the grid with a zero point, improving range alignment but changing the deployed format: kernels must handle the offset and store its metadata. We instead ask whether the signed alphabet's native asymmetry can be matched to the data while keeping $z = 0$. Signed symmetric quantization does exactly this. We first define the asymmetric and symmetric quantizers, then introduce the signed symmetric family and its two sign selection rules.

\paragraph{Asymmetric quantizer.} Asymmetric quantization allows both the scale and zero point to be free parameters (\textit{i.e.}, $s,z \in \mathbb{R}$).
A common choice is the standard min-max grid,
\begin{equation}
s = \frac{\max(w) - \min(w)}{2^q - 1}, \qquad
z = -\frac{\min(w)}{s},
\label{eq:minmax}
\end{equation}
on the unsigned alphabet $\mathcal{A}_q := \{0, \ldots, 2^q - 1\}$.

\paragraph{Symmetric quantizer.}
Symmetric quantization fixes the zero point to zero (\textit{i.e.}, $z = 0$) and commonly defines the scale as the maximum absolute value (or absmax),
\begin{equation}
s = \frac{\|w\|_\infty}{2^{q-1}},
\label{eq:absmax}
\end{equation}
on the signed alphabet $\mathcal{A}_q := \{-2^{q-1}, \ldots, 2^{q-1} - 1\}$.
We refer to this as the absmax grid.

\paragraph{Signed symmetric quantizer.}
We define signed symmetric quantization as the family of uniform integer quantizers on the signed alphabet $\mathcal{A}_q = \{-2^{q-1}, \ldots, 2^{q-1}-1\}$ with zero point $z = 0$ and non-zero scale (\textit{i.e.}, $s \neq 0$). Formally, for $\alpha > 0$ and $\gamma \in \{-1, +1\}$, let $s := \gamma\alpha$ and
\begin{equation}
\mathcal{Q}_\gamma(w) := \mathcal{Q}(w;\, \gamma\alpha,\, 0) = \mathcal{Q}(w;\, s,\, 0).
\label{eq:signed_sym_quant}
\end{equation}
Note that the standard symmetric quantizer is a degenerate case with $\gamma = +1$. Furthermore, allowing $\gamma = -1$ preserves the grid resolution and changes only which tail receives the extra representable endpoint.

\paragraph{The signed absmax grid.} Within this family, the signed absmax grid takes the absmax magnitude and selects its sign by the following signed absmax rule.\footnote{We do not claim this rule as novel: it already appears in some of the integer formats of \texttt{llama.cpp}~\cite{llama_cpp}. The strictly positive convention nonetheless remains standard, and we are not aware of prior analysis of the choice.} Let $M := \|w\|_\infty > 0$ and $\alpha := M / 2^{q-1}$. Choose
\begin{equation}
\gamma^\star = -\operatorname{sign}(w_{i^\star}),
\qquad i^\star \in \argmax_{i \in [d]} |w_i|,
\label{eq:sign_rule}
\end{equation}
where $i^\star$ indexes the largest-magnitude coordinate, with deterministic tie-breaking.
Intuitively, this sign choice maps the largest representable magnitude in $\mathcal{A}_q$ to the dominant outlier in $w$. In particular,
\[
\mathcal{Q}_{\gamma^\star}(w_{i^\star}) \;=\; \gamma^\star\alpha \cdot (-2^{q-1}) \;=\; -\gamma^\star M \;=\; \operatorname{sign}(w_{i^\star})\, M \;=\; w_{i^\star}.
\]
Hence, the dominant outlier is represented exactly. The opposite sign choice would instead place the extra endpoint on the other tail, potentially clipping the largest-magnitude coordinate. Indeed, in Section~\ref{sec:theoretical-analysis}, Corollary~\ref{cor:optimal-sign} establishes that this sign choice minimizes the worst-case error bound in Theorem~\ref{thm:worst-case-bound} under a condition that holds for 88--99\% of evaluated LLM weight groups (Remark~\ref{rem:cardinality-margin}).

\paragraph{The MSE-optimal sign rule.} We introduce a second sign selection rule that, at any fixed scale magnitude $\alpha$, directly minimizes $\|w - \mathcal{Q}_\gamma(w)\|_2^2$. Let $T := \bigl(2^{q-1} - \tfrac{1}{2}\bigr)\alpha$ denote the magnitude beyond which coordinates are clipped, and choose
\begin{equation}
\gamma^\dagger = -\operatorname{sign}\bigl(D(w)\bigr), \qquad D(w) := \sum_{i=1}^{d} \operatorname{sign}(w_i)\bigl(|w_i| - T\bigr)_+,
\label{eq:mse_sign_rule}
\end{equation}
where $(x)_+ := \max(x, 0)$, with either sign chosen when $D(w) = 0$.
Every clipped coordinate contributes to $D(w)$ in proportion to how far its magnitude exceeds $T$, so $D(w)$ compares the excess $\ell_1$ mass on the two tails and $\gamma^\dagger$ orients the grid toward the heavier one.
In Section~\ref{sec:theoretical-analysis}, we prove that $\gamma^\dagger$ minimizes $\|w - \mathcal{Q}_\gamma(w)\|_2^2$ for every weight vector and at any fixed scale magnitude $\alpha$ (Corollary~\ref{cor:mse-optimal}).
Empirically, we pair this rule with the absmax magnitude $\alpha = M / 2^{q-1}$ for consistency, so that the resulting grid and the signed absmax grid differ only in the sign they select.

Both rules share a key distinction from min-max asymmetric quantization: rather than covering the entire interval $[\min w, \max w]$, they use the asymmetry already present in the signed alphabet, placing the extra representable value on one tail.
Importantly, the choice between asymmetric and (possibly signed) symmetric quantization implies different inference costs~\cite{jacob2018quantization, gholami2022survey}, as we observe directly in GGUF models executed with \texttt{llama.cpp}~\cite{llama_cpp}.
At 4-bit precision, \texttt{Q4\_1} is defined as an asymmetric format per group of 32 elements, storing an FP16 scale, an FP16 zero point, and 32 4-bit weights per group; \texttt{Q4\_0} is the corresponding symmetric format, storing only an FP16 scale and 32 4-bit weights per group.
All three sign choices ($\gamma = +1$, $\gamma^\star$, and $\gamma^\dagger$) deploy in \texttt{Q4\_0} and differ only in the sign of the stored scale, leaving the kernel path unchanged.
We implement analogous 2-bit formats, \texttt{Q2\_1} and \texttt{Q2\_0}, to further characterize the low-bit regime; these are experimental formats implemented for this study, with implementation details and further discussion deferred to Appendix~\ref{sec:implementation}.
We quantify the inference cost of these formats in Section~\ref{sec:results} (Table~\ref{tbl:cost}).
Section~\ref{sec:theoretical-analysis} analyzes the sign selection rules, and Section~\ref{sec:results} evaluates how they translate into model-level accuracy.

\section{Theoretical Analysis}
\label{sec:theoretical-analysis}

In Section~\ref{sec:signed_sym}, we defined the signed symmetric quantization family and argued informally for two sign selection rules: the signed absmax rule aligns the representable endpoint with the dominant outlier (Equation~\ref{eq:sign_rule}), while the MSE-optimal rule orients the grid toward the tail carrying more excess mass (Equation~\ref{eq:mse_sign_rule}).
We now formalize both arguments.

We first define the clipped set (Definition~\ref{def:clipped-set}), then isolate the two roles of the scale: (1) its magnitude governs grid resolution and rounding error, while (2) its sign governs grid polarity and clipping error (Lemma~\ref{lem:error-separability}).
We bound the $\ell_2$ error in terms of the clipped set (Theorem~\ref{thm:worst-case-bound}), then prove that the signed absmax rule minimizes this bound under a mild cardinality condition (Corollary~\ref{cor:optimal-sign}).
We empirically verify this condition on real LLM weights and show that reducing the clipped set cardinality is an effective surrogate for reducing $\ell_2$ quantization error (Remark~\ref{rem:cardinality-margin}).
We then derive a closed-form rule that minimizes the exact $\ell_2$ error at any fixed scale magnitude (Corollary~\ref{cor:mse-optimal}), and compare the two rules with each other and against exhaustive search (Remark~\ref{rem:oracle}).
We close by showing that the signed symmetric quantization family is a non-trivial special case of the asymmetric quantization family (Theorem~\ref{thm:signed-asymmetric}).

Unless stated otherwise, we assume weight vector \( w \in \mathbb{R}^d \) is quantized to the $q$-bit signed integer alphabet $\mathcal{A}_q := \{-2^{q-1}, \ldots, 2^{q-1}-1\}$ under round-to-nearest with step size $\alpha := M / 2^{q-1}$, where $M := \|w\|_\infty > 0$. We note that, although rounding ties can occur, they form a measure-zero set under any continuous weight distribution; nonetheless, our analysis is largely invariant to how they are broken (Appendix~\ref{sec:proofs}). We state explicitly where this invariance does not hold (namely, Theorem~\ref{thm:signed-asymmetric}) and the condition required there. In practice, our implementation breaks ties to even.
For brevity, all proofs are deferred to Appendix~\ref{sec:proofs}.

\begin{definition}[Clipped Set]
\label{def:clipped-set}
Let \( T := \bigl(2^{q-1} - \tfrac{1}{2}\bigr)\alpha \).
For each \(\gamma \in \{-1,+1\}\), we define the \emph{clipped set}
\begin{equation}
\label{eq:clipped-set}
C_\gamma = \bigl\{ i \in [d] : \gamma w_i > T \bigr\},
\end{equation}
so that \(C_\gamma\) collects exactly the indices at which clipping under \(\mathcal{Q}_\gamma\) (Equation~\ref{eq:signed_sym_quant}) strictly increases the squared error relative to rounding on the unbounded grid. 
Its complement \([d] \setminus C_\gamma\) collects those at which it does not.
Away from the boundary \(\gamma w_i = T\), this is the set of indices whose rounded code leaves the alphabet, \(\lfloor w_i / (\gamma\alpha) \rceil \notin \mathcal{A}_q\). 
At the boundary, the rounded code may leave the alphabet but both candidate reconstructions lie at distance \(\alpha/2\), so clipping there carries no penalty and the index is excluded.
\end{definition}

\begin{restatable}[Error Decomposition]{lemma}{lemseparability}
\label{lem:error-separability}
Let $R(w_i) := \alpha\lfloor w_i/\alpha\rceil$ denote rounding to the unbounded grid $\alpha\mathbb{Z}$. Define rounding error \( E(w) \) and clipping penalty \( \Delta_\gamma(w) \) as
\[
E(w) := \sum_{i=1}^d \bigl(w_i - R(w_i)\bigr)^2, \qquad \Delta_\gamma(w) := \sum_{i \in C_\gamma} \Bigl[\bigl(w_i - \mathcal{Q}_\gamma(w_i)\bigr)^2 - \bigl(w_i - R(w_i)\bigr)^2\Bigr].
\]
Define $\alpha := M/2^{q-1}$, so that $T = M(1 - 2^{-q})$ by Definition~\ref{def:clipped-set}. 
Then $E(w)$ is independent of $\gamma$, and the $\ell_2$ error decomposes as
\begin{equation}
\label{eq:error-separability}
\|w - \mathcal{Q}_\gamma(w)\|_2^2 = E(w) + \Delta_\gamma(w) = E(w) + 2\alpha \sum_{i \in C_\gamma} (\gamma w_i - T).
\end{equation}
\end{restatable}

Lemma~\ref{lem:error-separability} separates the two roles of the scale: (1) its magnitude $\alpha$ determines the grid resolution and rounding error $E(w)$, and (2) its sign $\gamma$ determines the grid polarity and clipping penalty $\Delta_\gamma(w)$.
The following theorem converts this exact expression to a deterministic worst-case bound.

\begin{restatable}[Worst-Case Error Bound]{theorem}{thmworstcase}
\label{thm:worst-case-bound}
Given \( w \in \mathbb{R}^d \), each per-coordinate error term satisfies
\[
\bigl(w_i - \mathcal{Q}_\gamma(w_i)\bigr)^2 \le \tfrac{\alpha^2}{4} \;\;\text{for }\, i \in [d] \setminus C_\gamma, \qquad \tfrac{\alpha^2}{4} < \bigl(w_i - \mathcal{Q}_\gamma(w_i)\bigr)^2 \le \alpha^2 \;\;\text{for }\, i \in C_\gamma.
\]
In particular, substituting $\alpha = M / 2^{q-1}$ and writing $\rho_\gamma = |C_\gamma|/d$, the total $\ell_2$ error satisfies
\begin{equation}
\label{eq:bound-rewrite}
\|w - \mathcal{Q}_\gamma(w)\|_2^2 \;\le\; 2^{-2q}\,M^2\,d\,(1 + 3\rho_\gamma).
\end{equation}
\end{restatable}

Theorem~\ref{thm:worst-case-bound} turns the decomposition identified in Lemma~\ref{lem:error-separability} into a worst-case objective for sign selection. Given $w$, Lemma~\ref{lem:error-separability} shows that the sign of the scale affects the $\ell_2$ error only through clipping. Theorem~\ref{thm:worst-case-bound} then upper bounds the total squared error by a quantity whose only $\gamma$-dependent term is the clipped fraction $\rho_\gamma = |C_\gamma|/d$. Therefore, under the deterministic worst-case bound, sign selection reduces to minimizing the size of the clipped set.

Importantly, as outliers are well-known to dominate few-bit quantization error~\cite{ashkboos2024quarot,liu2025spinquant}, worst-case error bounds are effective optimization targets for post-training quantization~\cite{sanjeet2026mixquant}.
The following corollary and remark build on this intuition to justify the signed absmax rule (Equation~\ref{eq:sign_rule}), both analytically and empirically: Corollary~\ref{cor:optimal-sign} characterizes when this sign choice minimizes the worst-case bound, and Remark~\ref{rem:cardinality-margin} shows that this condition holds for 88--99\% of evaluated LLM weight groups across low bit widths, and that the cardinality margin is highly predictive of realized $\ell_2$ gains.

\begin{restatable}[Conditional Bound-Optimal Sign Choice]{corollary}{coroptimalsign}
\label{cor:optimal-sign}
Let \(i^\star \in \argmax_{i \in [d]} |w_i|\) index the dominant outlier in \(w\) and define
\[
\gamma^\star := -\operatorname{sign}(w_{i^\star}).
\]
Since the worst-case bound from Theorem~\ref{thm:worst-case-bound} depends on  \(\gamma\) only through \(|C_\gamma|\),
\[
\argmin_{\gamma \in \{-1,+1\}} 2^{-2q}\,M^2\,d\,(1 + 3\rho_\gamma) \;=\; \argmin_{\gamma \in \{-1,+1\}} |C_\gamma|.
\]
Then, whenever \(|C_{\gamma^\star}| \le |C_{-\gamma^\star}|\), \(\gamma^\star\) is a minimizer of the worst-case bound.
\end{restatable}

Corollary~\ref{cor:optimal-sign} links the signed absmax rule to the worst-case analysis in Theorem~\ref{thm:worst-case-bound}. In particular, whenever \(|C_{\gamma^\star}| \le |C_{-\gamma^\star}|\), the choice $\gamma^\star$ minimizes Equation~\ref{eq:bound-rewrite}.
This yields a directly interpretable optimization criterion: aligning the extra representable endpoint with the dominant outlier should not increase the number of clipped coordinates.
It also establishes the size of the clipped set $|C_\gamma|$ as a surrogate for the exact sign-dependent clipping penalty $\Delta_\gamma(w)$ in Lemma~\ref{lem:error-separability}.
Both relationships are directly measurable on LLM weight groups, allowing us to evaluate not only how often the condition holds, but also how well the cardinality margin predicts realized $\ell_2$ error reduction.

\begin{figure}[t]
    \centering
    \includegraphics[width=0.9\linewidth]{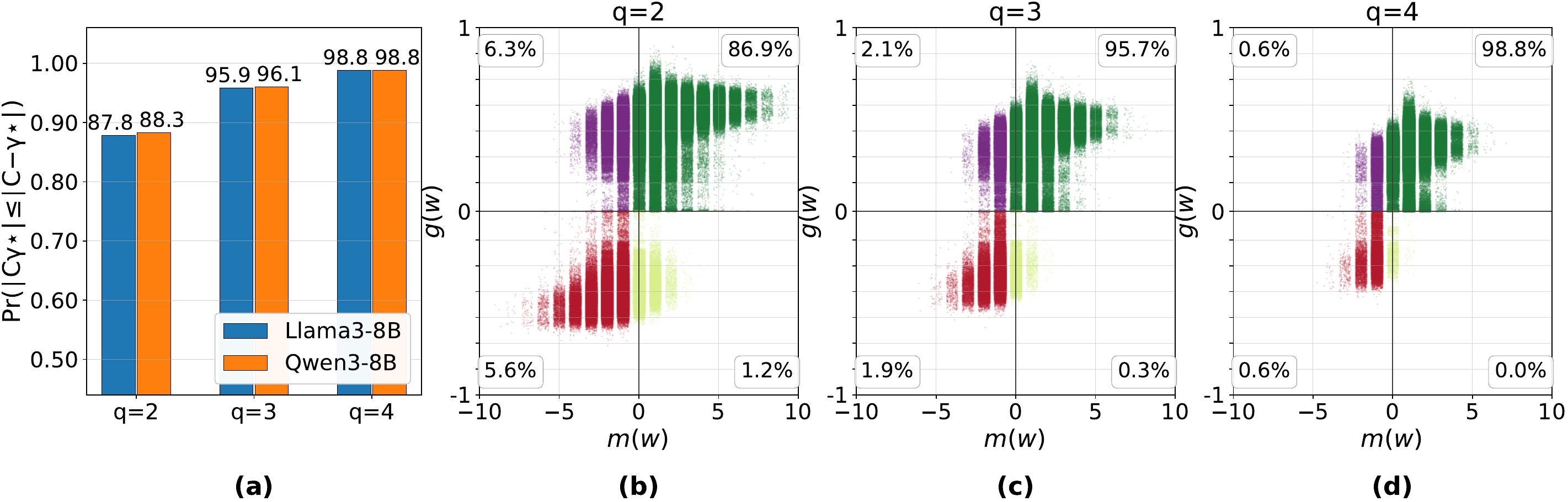}
    \caption{We empirically validate the cardinality condition (Corollary~\ref{cor:optimal-sign}) on Llama3 8B and Qwen3 8B, quantizing all linear projection weights with group size 32 (over 200M groups per model; no rotations or error correction). 
    \textbf{(a)} The bars show the per-model fraction of weight groups satisfying $|C_{\gamma^\star}| \le |C_{-\gamma^\star}|$ at each bit width $q$.
    \textbf{(b)--(d)} Each panel corresponds to one bit width $q \in \{2, 3, 4\}$.
    The panels show a pooled scatter of cardinality margin $m(w) = |C_{-\gamma^\star}| - |C_{\gamma^\star}|$ against realized error gain $g(w) = \Delta_{-\gamma^\star}(w) - \Delta_{\gamma^\star}(w)$ (Lemma~\ref{lem:error-separability}).
    The cardinality condition holds when $m \ge 0$; the sign rule strictly reduces $\ell_2$ error when $g > 0$.
    Colors encode the quadrant assignment of each point.
    Annotated percentages use full pooled counts, with $g = 0$ included in the upper half.
    The scatter displays a 1\% subsample with horizontal jitter for visual clarity.}
    \label{fig:theory-diagnostics}
    \end{figure}

\begin{remark}[Cardinality Margin and Realized $\ell_2$ Error]
\label{rem:cardinality-margin}
Figure~\ref{fig:theory-diagnostics} evaluates the cardinality condition in Corollary~\ref{cor:optimal-sign} on real LLM weight groups by comparing the cardinality margin
\[
m(w) = |C_{-\gamma^\star}| - |C_{\gamma^\star}|
\]
with the realized error gain
\[
g(w) = \|w - \mathcal{Q}_{-\gamma^\star}(w)\|_2^2 - \|w - \mathcal{Q}_{\gamma^\star}(w)\|_2^2 = \Delta_{-\gamma^\star}(w) - \Delta_{\gamma^\star}(w),
\]
where the last equality follows from the decomposition in Lemma~\ref{lem:error-separability}. Thus, $m(w) \ge 0$ is exactly the condition under which Corollary~\ref{cor:optimal-sign} shows that $\gamma^\star$ minimizes the worst-case bound, while $g(w) > 0$ means that $\gamma^\star$ also reduces the realized $\ell_2$ error. Figure~\ref{fig:theory-diagnostics}(a) shows that on Llama3 8B and Qwen3 8B, the cardinality condition holds for roughly $88\%$, $96\%$, and $99\%$ of evaluated groups of 32 weights at $q \in \{2, 3, 4\}$, respectively. We further evaluate the impact of group size in Appendix~\ref{sec:ablation-groupsize}.
Panels (b)--(d) further show that, when the cardinality condition holds ($m(w) \ge 0$), the $\ell_2$ error does not increase in $98.6\%$, $99.7\%$, and $100.0\%$ of cases at $q \in \{2, 3, 4\}$, respectively.
Thus, although substantial variation may exist within worst-case error bounds, the objective induced by Theorem~\ref{thm:worst-case-bound} is both analytically tractable and empirically predictive: on real LLM weights, reductions in the size of the clipped set consistently correspond to reductions in the exact clipping penalty and, via Lemma~\ref{lem:error-separability}, reductions in realized quantization error.
\end{remark}

We note that the choice $\gamma^\star = -\operatorname{sign}(w_{i^\star})$ also carries two structural properties that hold unconditionally, independent of whether the cardinality condition is met.
First, as detailed in Section~\ref{sec:signed_sym}, it represents the dominant outlier exactly: $\mathcal{Q}_{\gamma^\star}(w_{i^\star}) = w_{i^\star}$, since $|w_{i^\star}| = M$ coincides with the representable endpoint placed on its side.
Second, it confines clipping to the opposite tail: $C_{\gamma^\star} \subseteq \{i \in [d] : \gamma^\star w_i > 0\}$.
These properties help explain why this rule remains effective in practice even on the small fraction of groups where the condition fails (Figure~\ref{fig:theory-diagnostics}): preserving the dominant outlier exactly and confining clipping to the opposite tail make $\gamma^\star$ a natural choice when a clear dominant tail is present.
Tables~\ref{tbl:llama8b_full} and~\ref{tbl:llama8b_learned} corroborate this at the model level: the largest gains appear at 2 bits, precisely where the cardinality condition holds least often (Figure~\ref{fig:theory-diagnostics}).

A natural alternative is to select the sign exhaustively, by quantizing each group under both polarities and keeping whichever yields lower reconstruction error.
The following result derives this sign in closed form at any fixed scale magnitude.
It therefore achieves the same minimum without the cost of the search.

\begin{restatable}[Closed-Form MSE-Optimal Sign]{corollary}{cormseoptimal}
\label{cor:mse-optimal}
Fix any $\alpha > 0$ and let $T := \bigl(2^{q-1} - \tfrac{1}{2}\bigr)\alpha$. Define
\begin{equation}
\label{eq:mse-sign-statistic}
D(w) := \sum_{i=1}^{d} \operatorname{sign}(w_i)\bigl(|w_i| - T\bigr)_+,
\end{equation}
where $(x)_+ := \max(x, 0)$. For $D(w) \ne 0$, the sign
\[
\gamma^\dagger := -\operatorname{sign}\bigl(D(w)\bigr)
\]
minimizes $\|w - \mathcal{Q}_\gamma(w)\|_2^2$ over $\gamma \in \{-1,+1\}$. When $D(w) = 0$, both signs are optimal.
\end{restatable}

Corollary~\ref{cor:mse-optimal} establishes that $\gamma^\dagger$ directly minimizes the $\ell_2$ weight reconstruction error at any fixed magnitude $\alpha$.
Moreover, whenever $\alpha \ge M/2^{q-1}$, $\gamma^\dagger$ also minimizes the $\ell_1$ weight reconstruction error (Remark~\ref{rem:l1-optimal-sign}).
In contrast, Corollary~\ref{cor:optimal-sign} establishes that, at $\alpha = M/2^{q-1}$, $\gamma^\star$ minimizes an upper bound on the $\ell_2$ error under a cardinality condition.
The following remark compares the two rules, identifying when their sign choices agree, quantifying the error where they differ, and comparing their computation costs.

\begin{remark}[Comparing Sign Selection Rules]
\label{rem:oracle}
    The signed absmax rule $\gamma^\star$ selects polarity from the largest-magnitude coordinate, while the MSE-optimal rule $\gamma^\dagger$ selects it from the difference in excess $\ell_1$ mass beyond $T$ between the two tails.
    When $\alpha = M / 2^{q-1}$ and $D(w) \neq 0$, the two rules agree whenever $|C_{\gamma^\star}| \le 1$.
    This includes the case where no coordinate is clipped under $\gamma^\star$, and therefore the case where the dominant outlier is the only coordinate beyond either threshold.
    Note that the cardinality condition of Corollary~\ref{cor:optimal-sign} does not itself guarantee MSE-optimality, since it certifies the minimizer of an upper bound rather than of the exact error.
    For example, let $q = 4$ and $w = (16,\, 15.1,\, -15.9,\, -15.9)$, so that $M = 16$, $\alpha = 2$, and $T = 15$.
    Here $\gamma^\star = -1$ and $|C_{\gamma^\star}| = |C_{-\gamma^\star}| = 2$, yet $D(w) = -0.7$ gives $\gamma^\dagger = +1$.
    Empirically, $\gamma^\star$ agrees with the MSE-optimal sign on $93\%$, $98\%$, and $99.4\%$ of groups at $q \in \{2, 3, 4\}$, and the disagreements account for $0.8\%$, $0.2\%$, and $0.06\%$ excess reconstruction error, respectively.
    In contrast, the strictly positive convention (\textit{i.e.}, $\gamma=+1$) disagrees on roughly half of all groups at every bit width, no better than an unbiased Rademacher sign, and incurs close to $20\%$ excess error.
    Finally, across Llama3 8B and Qwen3 8B, selecting $\gamma^\star$ costs approximately the same as the standard absmax scale computation, while selecting $\gamma^\dagger$ costs up to $2.6\times$ as much; for reference, brute-force evaluation of both polarities costs up to $7.0\times$ as much.
\end{remark}

Taken together, Corollaries~\ref{cor:optimal-sign} and~\ref{cor:mse-optimal} give two closed-form rules for one degree of freedom: $\gamma^\star$ minimizes the worst-case bound wherever the cardinality condition holds, and $\gamma^\dagger$ minimizes the realized error at any fixed magnitude.
Both decide which tail receives the extra representable value rather than leaving that to convention, and Section~\ref{sec:results} measures what the decision is worth against the strictly positive convention.
Asymmetric quantization reaches a comparable alignment by a different route, introducing a zero point that carries a non-trivial inference cost~\cite{jacob2018quantization, jain2020trained, gholami2022survey} (Table~\ref{tbl:cost}).
The next theorem shows that signed symmetric quantization is analytically a non-trivial instance of the family of asymmetric quantizers. In particular, negating the scale is equivalent to a unit zero-point shift on the same signed integer alphabet.

\begin{restatable}[Sign Flip as Unit Zero-Point Shift]{theorem}{thmsignedasymmetric}
\label{thm:signed-asymmetric}
Let \(\alpha > 0\). The quantizers \(\mathcal{Q}(\cdot; -\alpha, 0)\) and \(\mathcal{Q}(\cdot; \alpha, -1)\) induce the same representable grid when sharing the same signed integer alphabet with a grid resolution of \(\alpha\). Moreover, \(\mathcal{Q}(w; -\alpha, 0) = \mathcal{Q}(w; \alpha, -1)\) for any \(w \in \mathbb{R}\) with \(w/\alpha \notin \mathbb{Z} + \tfrac{1}{2}\).
\end{restatable}

Theorem~\ref{thm:signed-asymmetric} shows that, at a shared step size $\alpha$, negating the scale under $z = 0$ is identical to setting $z = -1$ with a strictly positive scale on the same signed integer alphabet.
Signed symmetric quantization therefore admits equivalent views: (1) operationally a symmetric quantizer ($z = 0$) that avoids zero point overhead at inference, and (2) analytically an asymmetric quantizer with a unit zero point shift ($z = -1$) that absorbs the shift into the sign of the scale.

It is worth noting that the relationship between the standard min-max asymmetric grid (Equation~\ref{eq:minmax}) and the signed symmetric grids (Equations~\ref{eq:sign_rule} and~\ref{eq:mse_sign_rule}) is not addressed by Theorem~\ref{thm:signed-asymmetric}.
The min-max grid preserves both tails by covering $[\min w, \max w]$ exactly, eliminating clipping error and leaving rounding as the only source of error (Lemma~\ref{lem:error-separability}). In contrast, the signed absmax rule preserves only the dominant outlier, and the MSE-optimal sign rule surrenders even that when the opposite tail carries more excess mass.
Furthermore, the min-max grid uses an unsigned alphabet and a scale derived from the full data range, \textit{i.e.}, $\max w - \min w$, rather than from $\|w\|_\infty$, so its relationship with $\alpha$ varies with the distribution of $w$.
An analytical comparison between the signed symmetric and min-max asymmetric grids is therefore non-trivial.
However, we compare them empirically in Section~\ref{sec:results}, where both sign rules yield better downstream accuracy than the min-max asymmetric one at 2 bits on both Llama3 8B and Qwen3 8B.

\section{Experimental Results}
\label{sec:results}

\begin{table}[!t]
\centering
\caption{Inference cost of the symmetric grid relative to the asymmetric grid in \texttt{llama.cpp} on an AMD EPYC\texttrademark{} ``Turin'' 128-core CPU. Each cell reports the symmetric-to-asymmetric ratio of model memory footprint (GB), prefill throughput (tokens/s), and decode throughput (tokens/s) at the same precision. Values $<1$ for memory and $>1$ for throughput favor the symmetric format.}
\label{tbl:cost}
{\small
\begin{tabular}{cc ccccc cccc ccc}
\toprule
\multirow{2}{*}{Format} & \multirow{2}{*}{Metric} & \multicolumn{5}{c}{Qwen3} & \multicolumn{4}{c}{Qwen3.5} & \multicolumn{3}{c}{Llama~3} \\
\cmidrule(lr){3-7} \cmidrule(lr){8-11} \cmidrule(lr){12-14}
 & & 0.6B & 1.7B & 4B & 8B & 14B & 0.8B & 2B & 4B & 9B & 1B & 3B & 8B \\
\midrule
\multirow{3}{*}{\texttt{Q4\_0} / \texttt{Q4\_1}} & memory ($\downarrow$)  & 0.93 & 0.92 & 0.91 & 0.91 & 0.91 & 0.94 & 0.93 & 0.92 & 0.91 & 0.93 & 0.92 & 0.91 \\
                    & prefill ($\uparrow$) & 1.23 & 1.65 & 1.51 & 1.73 & 2.45 & 1.10 & 1.58 & 1.26 & 1.59 & 1.49 & 1.77 & 2.21 \\
                    & decode ($\uparrow$)  & 1.00 & 1.22 & 1.55 & 2.15 & 2.07 & 0.98 & 1.12 & 0.98 & 1.41 & 0.98 & 0.90 & 2.03 \\
\midrule
\multirow{3}{*}{\texttt{Q2\_0} / \texttt{Q2\_1}} & memory ($\downarrow$)  & 0.90 & 0.88 & 0.87 & 0.86 & 0.85 & 0.92 & 0.91 & 0.88 & 0.87 & 0.90 & 0.87 & 0.86 \\
                    & prefill ($\uparrow$) & 1.03 & 1.13 & 1.19 & 1.25 & 1.29 & 1.16 & 1.39 & 1.02 & 1.08 & 1.17 & 1.22 & 1.24 \\
                    & decode ($\uparrow$)  & 1.00 & 1.30 & 1.11 & 1.13 & 1.25 & 1.01 & 1.03 & 0.97 & 1.06 & 1.03 & 1.27 & 1.08 \\
\bottomrule
\end{tabular}
}
\end{table}

Sections~\ref{sec:signed_sym} and~\ref{sec:theoretical-analysis} establish the sign of the scale as a degree of freedom with a closed-form optimum: it decides which tail of the data receives the extra representable value of the signed integer alphabet, and therefore which tail is exposed to clipping. We now evaluate whether this local choice of grid polarity yields benefits for end-to-end model accuracy.

We organize our evaluation in two parts. First, we study sign selection in isolation, without rotations, error correction, or activation quantization, and show that both rules improve over the strictly positive absmax grid across twelve checkpoints (Table~\ref{tbl:sign_isolation}). Second, we study sign selection under composition, using Llama3 8B and Qwen3 8B as case studies to show that this improvement survives state-of-the-art transformations and rounding methods (Tables~\ref{tbl:llama8b_full} and~\ref{tbl:qwen8b_main}).

\paragraph{Models and datasets.} Our experiments span three model families: Qwen3~\cite{qwen3technicalreport}, Qwen3.5~\cite{qwen3.5}, and Llama3~\cite{grattafiori2024llama3}, all sourced as instruction fine-tuned checkpoints from Huggingface~\cite{wolf2020transformers} and used without modification. For calibration, we sample 128 random sequences of 2048 tokens from the Tulu3 dataset~\cite{lambert2025tulu}, and we report perplexity on the WikiText2~\cite{merity2017pointer} test split. To evaluate downstream reasoning, we report zero-shot accuracy via LightEval~\cite{lighteval} on ARC (Challenge and Easy)~\cite{arc}, HellaSwag~\cite{hellaswag}, PIQA~\cite{piqa}, and WinoGrande~\cite{WinoGrande}, which score a fixed set of candidate answers by likelihood. 
We further report MMLU-Redux (MMLU-R)~\cite{gema2025we} and GSM8K~\cite{cobbe2021training}, which instead require the model to generate its answer. For these two, we generate a single answer per question and report accuracy.
We limit generation to at most 2048 output tokens and otherwise follow the non-thinking generation protocol described for Qwen3 models in~\cite{qwen3technicalreport}, with temperature $0.7$, top-$p$ $0.8$, top-$k$ $20$.
We additionally report the expected acceptance rate (EAR), which measures distribution overlap between the quantized model and the BF16 reference, computed over the BF16 model's top $20$ tokens~\cite{helcig2026statistically}.
We quantize all models using the Brevitas quantization library~\cite{brevitas}.

\paragraph{Inference cost.} Before turning to accuracy, we quantify the deployment cost of the formats defined in Section~\ref{sec:signed_sym}, which is independent of the sign rule.
Table~\ref{tbl:cost} shows that the symmetric formats use less memory on larger models (up to $9\%$ at \texttt{Q4} and up to $15\%$ at \texttt{Q2}) and, on an AMD EPYC\texttrademark{} ``Turin'' CPU, consistently improve prefill throughput ($1.10$--$2.45\times$ at \texttt{Q4} and $1.02$--$1.39\times$ at \texttt{Q2}), with decode throughput mostly following the same trend.
To confirm that this advantage is cross-platform, we repeat the same protocol on GPU backends.
On an AMD Radeon\texttrademark{} AI Pro R9700, \texttt{Q4\_0} delivers $1.06$--$1.07\times$ the prefill throughput of \texttt{Q4\_1} and $1.03$--$1.06\times$ the decode throughput across the evaluated Qwen3 models.
Appendix~\ref{sec:impl-runtime} repeats the measurement on an AMD Instinct\texttrademark{} MI300X GPU, where the prefill advantage persists but the decode difference falls within run-to-run variation.
This is the price of the flexibility of the asymmetric format, and the signed symmetric grids avoid it while reclaiming a degree of freedom.

\paragraph{Sign selection in isolation.}
We first isolate the effect of sign selection by direct casting all linear projection weights with group size 32, matching the \texttt{llama.cpp} block convention. 
This experiment uses no rotations, no error correction methods (\textit{e.g.}, Qronos), no activation quantization, and no calibration or optimization. 
Thus, the three rows of each block in Table~\ref{tbl:sign_isolation} share one scale magnitude, the absmax $\alpha := \|w\|_\infty / 2^{q-1}$, and differ only in how its sign is set: the strictly positive convention ($\gamma = +1$), the signed absmax rule ($\gamma^\star$), and the MSE-optimal rule ($\gamma^\dagger$).

\begin{table}[!t]
\centering
\caption{Sign selection in isolation. WikiText2 perplexity ($\downarrow$) and zero-shot recovery ($\uparrow$) under direct casting at group size 32. We compare the strictly positive convention ($\gamma = +1$) against the conditionally bound-optimal sign rule ($\gamma = \gamma^\star$, Corollary~\ref{cor:optimal-sign}) and the closed-form MSE-optimal sign rule ($\gamma = \gamma^\dagger$, Corollary~\ref{cor:mse-optimal}). Zero-shot recovery is the geometric mean of the per-task quantized-to-BF16 accuracy ratios over ARC (Challenge and Easy), HellaSwag, WinoGrande, and PIQA. BF16 perplexity baselines are shown ($q = 16$). The expected acceptance rate for these same models and settings is reported in Table~\ref{tbl:ear_isolation}, Appendix~\ref{sec:ablation-ear}.}
\label{tbl:sign_isolation}
{\small
\begin{tabular}{ccc ccccc cccc ccc}
\toprule
\multirow{2}{*}{Metric} & \multirow{2}{*}{$q$} & \multirow{2}{*}{$\gamma$} & \multicolumn{5}{c}{Qwen3} & \multicolumn{4}{c}{Qwen3.5} & \multicolumn{3}{c}{Llama3} \\
\cmidrule(lr){4-8} \cmidrule(lr){9-12} \cmidrule(lr){13-15}
 & & & 0.6B & 1.7B & 4B & 8B & 14B & 0.8B & 2B & 4B & 9B & 1B & 3B & 8B \\
\midrule
\multirow{7}{*}{\shortstack{WikiText2\\perplexity ($\downarrow$)}}
 & 16 & --- & 18.5 & 15.1 & 12.1 & 8.7 & 7.7 & 15.4 & 10.9 & 8.6 & 7.9 & 11.8 & 9.1 & 6.5 \\
\cmidrule(lr){2-15}
 & \multirow{3}{*}{4} & $+1$ & 21.0 & 25.6 & \textbf{12.6} & 9.1 & 7.9 & 18.7 & 12.4 & 8.9 & 9.0 & 14.0 & 9.8 & 7.0 \\
 & & $\gamma^\star$ & 20.8 & 18.5 & \textbf{12.6} & \textbf{8.9} & \textbf{7.8} & \textbf{18.1} & \textbf{11.9} & \textbf{8.7} & 8.7 & 13.4 & \textbf{9.6} & \textbf{6.8} \\
 & & $\gamma^\dagger$ & \textbf{20.7} & \textbf{18.4} & \textbf{12.6} & \textbf{8.9} & \textbf{7.8} & 18.2 & \textbf{11.9} & \textbf{8.7} & \textbf{8.6} & \textbf{13.3} & \textbf{9.6} & \textbf{6.8} \\
\cmidrule(lr){2-15}
 & \multirow{3}{*}{3} & $+1$ & 40.8 & 81.6 & 17.4 & 11.0 & 9.3 & 29.3 & 23.8 & 12.8 & 10.9 & 32.3 & 15.6 & 9.8 \\
 & & $\gamma^\star$ & 35.6 & 25.9 & \textbf{17.0} & 10.2 & \textbf{8.6} & 26.4 & 16.9 & \textbf{11.7} & \textbf{9.5} & 21.9 & \textbf{12.4} & \textbf{8.6} \\
 & & $\gamma^\dagger$ & \textbf{35.2} & \textbf{25.8} & 17.4 & \textbf{10.1} & \textbf{8.6} & \textbf{26.2} & \textbf{16.7} & 11.9 & \textbf{9.5} & \textbf{21.7} & \textbf{12.4} & 8.7 \\
\midrule
\multirow{6}{*}{\shortstack{Zero-shot\\recovery (\%, $\uparrow$)}}
 & \multirow{3}{*}{4} & $+1$ & 94.3 & \textbf{95.3} & 96.6 & 97.2 & 98.5 & 95.1 & 95.6 & 98.5 & 100.4 & 95.3 & 97.2 & 98.1 \\
 & & $\gamma^\star$ & 95.6 & 95.1 & \textbf{97.2} & 98.6 & \textbf{99.8} & \textbf{95.7} & \textbf{97.2} & \textbf{99.1} & 100.6 & \textbf{95.8} & \textbf{98.5} & \textbf{98.8} \\
 & & $\gamma^\dagger$ & \textbf{95.7} & \textbf{95.3} & 96.9 & \textbf{98.8} & 99.7 & \textbf{95.7} & \textbf{97.2} & \textbf{99.1} & \textbf{100.8} & 95.7 & 98.3 & \textbf{98.8} \\
\cmidrule(lr){2-15}
 & \multirow{3}{*}{3} & $+1$ & 81.7 & 81.9 & 87.2 & 89.0 & 92.2 & 86.7 & 83.7 & 93.2 & 94.5 & 84.8 & 86.9 & 86.5 \\
 & & $\gamma^\star$ & 83.8 & 86.3 & \textbf{89.3} & \textbf{94.1} & \textbf{94.5} & \textbf{88.4} & 89.5 & 94.8 & 96.3 & \textbf{89.6} & 93.0 & \textbf{92.7} \\
 & & $\gamma^\dagger$ & \textbf{85.0} & \textbf{86.9} & 88.5 & 93.9 & 94.2 & 87.8 & \textbf{89.6} & \textbf{95.0} & \textbf{96.4} & 89.5 & \textbf{93.5} & \textbf{92.7} \\
\bottomrule
\end{tabular}
}
\end{table}

The results show that both sign rules improve consistently on the strictly positive convention. The signed absmax rule $\gamma^\star$ lowers WikiText2 perplexity relative to $\gamma = +1$ in 23 of the 24 configurations in Table~\ref{tbl:sign_isolation} and matches it in the remaining one, with gains concentrated at lower precision. In particular, the largest reductions occur on Qwen3 1.7B ($81.6$ to $25.9$) and Llama3 1B ($32.3$ to $21.9$) at W3. The MSE-optimal rule $\gamma^\dagger$ minimizes the exact $\ell_2$ error by construction (Corollary~\ref{cor:mse-optimal}), and attains the lowest perplexity of the three rules in 20 of the 24 configurations.

The downstream zero-shot recovery results, computed over ARC (Challenge and Easy), HellaSwag, WinoGrande, and PIQA, follow the same trend. Relative to the strictly positive convention, $\gamma^\star$ improves zero-shot recovery by $0.8$ points at W4 and $3.7$ points at W3 on average, improving in 23 of the 24 configurations and trailing in one by $0.2$ points; $\gamma^\dagger$ improves in 23 and ties in one. These results support the central claim from Section~\ref{sec:theoretical-analysis}: the sign of the scale is a small local choice but, at few-bit precision, it has measurable model-level consequences.

Appendix~\ref{sec:ablations} reports additional ablations, including a random sign baseline and a sweep over group size (Table~\ref{tbl:gs_perplexity}).
The random baseline confirms that the benefit comes from the selection rule itself rather than from the added degree of freedom alone, and the sweep confirms that this benefit persists across group sizes and bit widths.

\paragraph{Composition with PTQ methods.}
Direct casting isolates sign selection, but practical quantization pipelines typically include transformations, error correction, and activation quantization. For example, in \texttt{llama.cpp}, few-bit weight formats are commonly paired with dynamic \texttt{Q8\_0} activation quantization at inference. We therefore evaluate Llama3 8B as a case study under a practical PTQ pipeline. We first insert normalized Hadamard rotations into rotation-invariant regions within the model~\cite{ashkboos2024quarot} (Figure~\ref{fig:rotation-graph}). Weight grids are then calculated on the rotated weights, while activation quantization grids are computed dynamically from the rotated activations under \texttt{Q8\_0} quantization. Qronos is then applied as the adaptive rounding step to reduce the mismatch between the quantized model and its full-precision counterpart~\cite{zhang2026qronos}. We defer further algorithmic details to Appendix~\ref{sec:experiment-details}.

Table~\ref{tbl:llama8b_full} presents results for the symmetric \texttt{Q4\_0} under all three sign choices and for the asymmetric \texttt{Q4\_1}, with analogous notation for \texttt{Q3} and \texttt{Q2}.\footnote{\texttt{Q4\_0} and \texttt{Q4\_1} are standard GGUF formats, and the \texttt{Q2} variants are implemented for this study (Appendix~\ref{sec:implementation}). The \texttt{Q3} variants are simulated in Brevitas and are not deployed in \texttt{llama.cpp}, so we report their accuracy but not their inference cost.}
At W4, all four grids are close in accuracy. Both signed rules and \texttt{Q4\_1} reach $6.7$ WikiText2 perplexity, while the strictly positive convention gives $6.8$. Downstream recovery is similarly saturated: $99.1\%$ for \texttt{Q4\_1}, $98.8\%$ under either signed rule, and $98.2\%$ under $\gamma = +1$. In this regime, the choice between symmetric and asymmetric formats is dominated less by accuracy than by deployment cost. Since all three sign choices share the same runtime path, they inherit the memory and throughput advantages of the symmetric format measured in Table~\ref{tbl:cost}. On Llama3 8B at W4, this corresponds to a $2.21\times$ prefill and $2.03\times$ decode throughput advantage over \texttt{Q4\_1}.

\begin{table}[!t]
\centering
\caption{Composition with PTQ methods. WikiText2 perplexity ($\downarrow$), expected acceptance rate ($\uparrow$), downstream accuracy (\%, $\uparrow$), and recovery ($\uparrow$) on Llama3 8B for few-bit GGUF formats. All quantized rows use normalized Hadamard rotations and Qronos. The symmetric format is reported under three sign choices: the strictly positive convention ($\gamma = +1$), the conditionally bound-optimal $\gamma^\star$ (Corollary~\ref{cor:optimal-sign}), and the closed-form MSE-optimal $\gamma^\dagger$ (Corollary~\ref{cor:mse-optimal}). Expected acceptance rate is computed over the BF16 model's top $20$ tokens on WikiText2. Recovery is the geometric mean of accuracy relative to BF16 over the seven reported tasks. Bold marks the best value in each column; underline marks the best symmetric value where the asymmetric format is better.}
\label{tbl:llama8b_full}
{\small\setlength{\tabcolsep}{4.5pt}
\begin{tabular}{ll cccccccccc}
\toprule
Format & $\gamma$ & Wiki2 $\downarrow$ & EAR \% & ARC-C & ARC-E & HellaS & WinoG & PIQA & MMLU-R & GSM8K & Recovery \% \\
\midrule
BF16 & --- & 6.5 & --- & 53.8 & 82.4 & 59.7 & 71.0 & 80.1 & 67.1 & 79.5 & --- \\
\midrule
\texttt{Q4\_0} & $+1$ & 6.8 & 92.4 & 51.1 & 81.4 & 58.6 & \underline{70.8} & 79.5 & 65.8 & \textbf{78.5} & 98.2 \\
\texttt{Q4\_0} & $\gamma^\star$ & \textbf{6.7} & 93.5 & 51.8 & 81.6 & 59.3 & 70.7 & \underline{80.1} & 66.2 & 78.3 & \underline{98.8} \\
\texttt{Q4\_0} & $\gamma^\dagger$ & \textbf{6.7} & \underline{93.6} & \underline{52.0} & \textbf{81.9} & \textbf{59.5} & 70.6 & 79.8 & \textbf{66.5} & 77.9 & \underline{98.8} \\
\texttt{Q4\_1} & --- & \textbf{6.7} & \textbf{94.1} & \textbf{52.8} & 81.3 & 59.4 & \textbf{71.7} & \textbf{80.4} & 66.2 & 77.7 & \textbf{99.1} \\
\midrule
\texttt{Q3\_0} & $+1$ & 7.7 & 82.9 & 46.8 & 79.3 & 55.6 & 69.1 & 78.1 & 61.2 & 74.3 & 93.6 \\
\texttt{Q3\_0} & $\gamma^\star$ & \underline{7.4} & 85.7 & \textbf{51.5} & 80.2 & \underline{57.4} & \textbf{70.1} & \textbf{79.4} & \textbf{63.9} & 74.7 & \textbf{96.5} \\
\texttt{Q3\_0} & $\gamma^\dagger$ & \underline{7.4} & \underline{86.0} & 50.7 & \textbf{81.1} & 57.0 & 69.9 & 78.2 & 63.4 & \underline{75.0} & 96.1 \\
\texttt{Q3\_1} & --- & \textbf{7.3} & \textbf{86.9} & 49.9 & 80.6 & \textbf{57.6} & 69.1 & 78.7 & 63.6 & \textbf{76.1} & 96.2 \\
\midrule
\texttt{Q2\_0} & $+1$ & 22.3 & 49.3 & 30.7 & 61.7 & 39.7 & 57.7 & 67.3 & 36.2 & 38.9 & 65.4 \\
\texttt{Q2\_0} & $\gamma^\star$ & 15.3 & 56.6 & \textbf{35.3} & 67.4 & 44.5 & 59.9 & 72.3 & 41.0 & 50.1 & 73.6 \\
\texttt{Q2\_0} & $\gamma^\dagger$ & \textbf{15.0} & \textbf{58.0} & 35.1 & \textbf{69.1} & \textbf{45.4} & \textbf{61.8} & \textbf{72.4} & \textbf{42.5} & \textbf{51.6} & \textbf{75.1} \\
\texttt{Q2\_1} & --- & \textbf{15.0} & 56.4 & 32.7 & 61.4 & 45.1 & 59.3 & 71.2 & 39.7 & 47.2 & 70.8 \\
\bottomrule
\end{tabular}
}
\end{table}

\begin{table}[!t]
\centering
\caption{Composition with PTQ methods on a second model family. WikiText2 perplexity ($\downarrow$), expected acceptance rate ($\uparrow$), downstream accuracy (\%, $\uparrow$), and recovery ($\uparrow$) on Qwen3 8B, under the protocol of Table~\ref{tbl:llama8b_full}. Bold marks the best value in each column; underline marks the best symmetric value where the asymmetric format is better.}
\label{tbl:qwen8b_main}
{\small\setlength{\tabcolsep}{4.5pt}
\begin{tabular}{ll cccccccccc}
\toprule
Format & $\gamma$ & Wiki2 $\downarrow$ & EAR \% & ARC-C & ARC-E & HellaS & WinoG & PIQA & MMLU-R & GSM8K & Recovery \% \\
\midrule
BF16 & --- & 8.7 & --- & 55.5 & 83.8 & 57.1 & 64.1 & 76.3 & 79.8 & 92.7 & --- \\
\midrule
\texttt{Q4\_0} & $+1$ & 8.9 & 93.1 & \textbf{56.6} & \textbf{83.2} & 56.0 & 63.9 & 75.8 & 78.8 & 91.9 & \textbf{99.5} \\
\texttt{Q4\_0} & $\gamma^\star$ & \textbf{8.8} & \underline{94.0} & 55.3 & 82.7 & \underline{56.3} & 63.5 & 76.2 & 79.1 & \textbf{92.9} & 99.3 \\
\texttt{Q4\_0} & $\gamma^\dagger$ & 8.9 & \underline{94.0} & 55.3 & 82.6 & \underline{56.3} & \underline{64.1} & \textbf{77.0} & \underline{79.3} & 92.3 & \textbf{99.5} \\
\texttt{Q4\_1} & --- & \textbf{8.8} & \textbf{94.4} & 54.1 & 82.6 & \textbf{56.6} & \textbf{64.2} & 76.5 & \textbf{79.6} & 92.8 & 99.3 \\
\midrule
\texttt{Q3\_0} & $+1$ & 9.6 & 86.0 & 52.6 & \underline{81.4} & 53.3 & \textbf{64.2} & 75.0 & 76.7 & 90.5 & 96.8 \\
\texttt{Q3\_0} & $\gamma^\star$ & \underline{9.4} & 87.8 & \underline{53.1} & 81.2 & 54.5 & 62.7 & \textbf{76.2} & 76.8 & 90.5 & \underline{97.1} \\
\texttt{Q3\_0} & $\gamma^\dagger$ & \underline{9.4} & \underline{87.9} & 50.8 & 80.1 & \textbf{54.6} & \textbf{64.2} & 76.0 & \underline{77.4} & \underline{91.1} & 96.7 \\
\texttt{Q3\_1} & --- & \textbf{9.2} & \textbf{88.4} & \textbf{53.6} & \textbf{81.9} & \textbf{54.6} & 62.9 & 75.5 & \textbf{78.8} & \textbf{91.8} & \textbf{97.8} \\
\midrule
\texttt{Q2\_0} & $+1$ & 17.3 & 61.0 & 32.8 & 62.7 & 39.7 & 55.6 & 67.0 & 54.7 & 66.8 & 73.4 \\
\texttt{Q2\_0} & $\gamma^\star$ & 13.7 & 69.0 & \textbf{39.5} & \textbf{71.4} & \textbf{44.0} & \textbf{57.2} & 69.7 & 59.4 & 76.8 & 81.3 \\
\texttt{Q2\_0} & $\gamma^\dagger$ & \textbf{13.2} & \textbf{69.9} & 38.1 & 69.8 & \textbf{44.0} & 56.7 & \textbf{70.9} & \textbf{62.2} & \textbf{81.0} & \textbf{81.9} \\
\texttt{Q2\_1} & --- & 13.7 & 67.8 & 32.6 & 60.7 & 42.9 & 55.2 & 68.7 & 58.6 & 76.1 & 76.2 \\
\bottomrule
\end{tabular}
}
\end{table}

At lower bit widths the sign choice becomes decisive, and the signed grids first catch up with and then overtake the min-max asymmetric grid on downstream recovery. We report small differences between recovery point estimates without claiming statistical significance. At W3, $\gamma^\star$ improves the strictly positive convention from $7.7$ to $7.4$ perplexity and from $93.6\%$ to $96.5\%$ recovery, compared with $96.2\%$ for \texttt{Q3\_1}, although \texttt{Q3\_1} retains the lower perplexity ($7.3$). At W2 the margin widens: the strictly positive convention degrades to $22.3$ perplexity and $65.4\%$ recovery, while $\gamma^\dagger$ reaches $15.0$ perplexity and $75.1\%$ recovery, matching \texttt{Q2\_1} on perplexity and exceeding it on recovery by $4.3$ points.
Table~\ref{tbl:qwen8b_main} repeats this pipeline on Qwen3 8B and reproduces the pattern. At W4 the signed rules and \texttt{Q4\_1} have similar recovery ($99.3\%$ and $99.5\%$ against $99.3\%$), and at W2 they reach $81.3\%$ and $81.9\%$ recovery against $76.2\%$ for \texttt{Q2\_1}. At W3, \texttt{Q3\_1} yields $0.7$-point and $1.1$-point higher recovery than $\gamma^\star$ and $\gamma^\dagger$, respectively.
The result survives other pipelines as well. Replacing rotations and error correction with SmoothQuant~\cite{xiao2023smoothquant} (Table~\ref{tbl:qwen8b_smoothquant}, Appendix~\ref{sec:ablation-smoothquant}), \texttt{Q3\_0} under $\gamma^\star$ recovers more than \texttt{Q3\_1} at W3 ($93.9\%$ versus $90.4\%$). Composing instead with learned grid parameters~\cite{esser2019learned, bhalgat2020lsq+} and learned rotations~\cite{liu2025spinquant} preserves the advantage of a signed grid over the strictly positive convention, and recomputing the grid at every forward pass improves on learning it (Appendices~\ref{sec:ablation-learned-llama} and~\ref{sec:ablation-grid}).

Taken together, our results support a consistent interpretation. 
The sign of the scale is a local degree of freedom with measurable model-level effects when clipping is material. 
Direct casting isolates this effect: both rules consistently improve over the strictly positive absmax grid (Table~\ref{tbl:sign_isolation}), and the improvement persists across group sizes (Table~\ref{tbl:gs_perplexity}, Appendix~\ref{sec:ablation-groupsize}). In practical PTQ pipelines, the signed symmetric grids deploy through exactly the symmetric runtime path while often matching or exceeding the accuracy of the min-max asymmetric grid. 
On both models, recovery is similar to the min-max grid at W4 and W3, while both sign rules yield higher recovery at W2. 

While the two rules agree on $93$--$99.4\%$ of groups (Remark~\ref{rem:oracle}), where they differ is instructive.
Under RTN, the signed absmax rule spends the extra representable value on the dominant outlier, so it preserves that weight exactly in every group.
The MSE-optimal rule instead orients the grid toward the tail carrying the most excess mass.
It preserves the dominant outlier as well whenever it is the only clipped coordinate, and often agrees with $\gamma^\star$ well beyond that case: the two differ only when the tail opposite the dominant outlier carries strictly more excess mass beyond $T$ than the outlier's own tail.
To build intuition for these differences, we consider how weight quantization error affects single-layer outputs on unseen data.
Define $r_\gamma := w - \mathcal{Q}_\gamma(w)$ as the weight quantization error, $\Sigma := \mathbb{E}_z[zz^\top]$ as the second moment of an unseen input $z$, and $\widehat{\Sigma}$ as its estimate on the calibration set.
Zhang et al.~\yrcite{zhang2025provable} (Remark 3.8) then decompose
\begin{equation}
\mathbb{E}_z\bigl|z^\top r_\gamma\bigr|^2 \;=\; \underbrace{r_\gamma^\top \widehat{\Sigma}\, r_\gamma}_{\text{calibration fit}} \;+\; \underbrace{r_\gamma^\top \bigl(\Sigma - \widehat{\Sigma}\bigr) r_\gamma}_{\text{generalization gap}}.
\label{eq:generalization}
\end{equation}
The second term is upper-bounded by $\|\Sigma - \widehat{\Sigma}\|_2 \|r_\gamma\|_2^2$ and, at a fixed scale magnitude, the sign of the scale is the only free parameter in $\|r_\gamma\|_2$.
Sign selection therefore acts on this bound, which limits how far single-layer output reconstruction error on unseen data can depart from the error over the calibration set.
Because $\gamma^\dagger$ minimizes $\|r_\gamma\|_2$ exactly (Corollary~\ref{cor:mse-optimal}), it directly limits the worst-case generalization gap under RTN.
These guarantees need not hold after adaptive rounding; empirically, our evaluation shows the MSE-optimal sign rule gives higher W2 recovery in both case studies, but neither rule uniformly gives the best downstream recovery across bit widths.
Either way, the sign remains a useful degree of freedom however the grid is obtained, at no extra inference cost relative to the conventional symmetric format.

\section{Discussion and Conclusion}
\label{sec:conclusion}

Standard symmetric quantization constrains the scale factor to be positive by convention. We showed that allowing signed scales exposes a zero-overhead degree of freedom: the sign of the scale selects which tail of the data receives the extra representable value of the signed integer alphabet.
Our analysis separates the role of scale magnitude, which controls rounding error, from the role of scale sign, which controls clipping. This separation yields three results.
First, placing the extra representable value on the dominant-outlier tail is conditionally bound-optimal on $\ell_2$ quantization error (Corollary~\ref{cor:optimal-sign}); we observe the condition holds for $88$--$99\%$ of evaluated weight groups.
Second, our new closed-form rule selects the sign that minimizes the exact $\ell_2$ error at any fixed scale magnitude (Corollary~\ref{cor:mse-optimal}).
Third, negating the scale of a standard symmetric quantizer is analytically equivalent to a unit zero point shift on the same signed integer alphabet (Theorem~\ref{thm:signed-asymmetric}).
Empirically, both sign rules improve on the conventional symmetric grid at few bits, and even overtake the min-max asymmetric grid on downstream recovery at 2 bits on both Llama3 8B and Qwen3 8B when composed with rotations and Qronos.
Section~\ref{sec:results} offers intuition for the differences between the rules: under round-to-nearest, the MSE-optimal sign rule minimizes an upper bound on the output-reconstruction gap between the calibration set and unseen data, while the signed absmax rule preserves the dominant outlier exactly.

\paragraph{Limitations and future work.}
Our theoretical analysis focuses mainly on data-free weight quantization. This is intentionally narrow: it isolates grid polarity as the degree of freedom at a fixed magnitude and yields closed-form rules that already capture much of the benefit observed in our PTQ experiments. 
The two guarantees differ in scope: Corollary~\ref{cor:optimal-sign} is established at the absmax magnitude, while Corollary~\ref{cor:mse-optimal} holds at any fixed magnitude. 
Extending our analysis to joint sign and magnitude selection is future work, as is evaluating signed scales across additional model families, operator types such as convolutions, quantization granularities, and deployment backends. Our results nevertheless establish the sign of the scale as a consequential choice: it controls grid polarity, admits closed-form selection, and improves few-bit model quality while retaining the runtime profile of symmetric quantization.

\bibliographystyle{mlsys2025}
\bibliography{references}

\clearpage
\appendix

{
  \centering
  {\LARGE\bfseries Appendices\par}
  \vspace{1.2em}
}
\etocsettocdepth.toc{subsection}
\etocsettocstyle{}{}
\tableofcontents
\vspace{1.5em}

\section{Proofs}
\label{sec:proofs}

Throughout this appendix, we adopt the setting of Section~\ref{sec:theoretical-analysis}: the signed integer alphabet $\mathcal{A}_q := \{-2^{q-1}, \ldots, 2^{q-1}-1\}$, step size $\alpha$, and round-to-nearest with ties broken to even.
A midpoint condition need not be imposed, except in Theorem~\ref{thm:signed-asymmetric}, which states one explicitly; the set it excludes has measure zero under any continuous weight distribution.
We use two properties.
The first is algebraic: $w/(\gamma\alpha) = \gamma w/\alpha$ for $\gamma \in \{-1,+1\}$, since $1/\gamma = \gamma$.
The second is that round-to-nearest-even is odd, $\lfloor -x\rceil = -\lfloor x\rceil$ for all $x \in \mathbb{R}$.
Away from ties this is immediate.
At a tie, $x$ lies halfway between consecutive integers $k$ and $k+1$; negation maps this pair to $-k-1$ and $-k$, and since $n$ and $-n$ share parity, the even member of the pair maps to the even member of the image.
Under a different round-to-nearest convention this second property can fail at ties.
For instance, round-half-up gives $\lfloor 2.5\rceil = 3$ but $\lfloor -2.5\rceil = -2$.
The candidates nonetheless remain equidistant, so the squared errors in Lemma~\ref{lem:error-separability} and Theorem~\ref{thm:worst-case-bound} are unchanged and both results hold as stated.

\subsection{Proof of Lemma~\ref{lem:error-separability}}

\lemseparability*

\begin{proof}
The map $R$ does not depend on $\gamma$, so $E(w)$ is $\gamma$-invariant by construction.
Set $u_i := \gamma w_i$. Since $1/\gamma = \gamma$, we have $\lfloor w_i/(\gamma\alpha)\rceil = \lfloor u_i/\alpha\rceil$; and since round-to-nearest-even is odd, $\lfloor w_i/\alpha\rceil = \gamma\lfloor u_i/\alpha\rceil$. Hence
\begin{align*}
\mathcal{Q}_\gamma(w_i) &= \gamma\alpha\,\clip\bigl(\lfloor u_i/\alpha\rceil;\, \mathcal{A}_q\bigr), \\
R(w_i) &= \gamma\alpha\,\lfloor u_i/\alpha\rceil.
\end{align*}

\emph{Step 1: $(w_i - \mathcal{Q}_\gamma(w_i))^2 = (w_i - R(w_i))^2$ for $i \notin C_\gamma$.}
By Definition~\ref{def:clipped-set}, $i \notin C_\gamma$ means $\gamma w_i \le T$, so $\lfloor u_i/\alpha\rceil \in \mathcal{A}_q$, the clip is inactive, and $\mathcal{Q}_\gamma(w_i) = R(w_i)$.
The sole exception is the boundary $\gamma w_i = T$, where the rounded value ties and may be clipped; all candidate values there lie at distance $\alpha/2$ from $w_i$, so the squared errors agree and the decomposition is unaffected.

\emph{Step 2: $(w_i - \mathcal{Q}_\gamma(w_i))^2 - (w_i - R(w_i))^2 = 2\alpha(\gamma w_i - T)$ for $i \in C_\gamma$.}
For $i \in C_\gamma$, $u_i \in (M - \alpha/2, M]$, so $\lfloor u_i/\alpha\rceil = 2^{q-1}$ and the clip caps it at $2^{q-1}-1$, giving $R(w_i) = \gamma M$ and $\mathcal{Q}_\gamma(w_i) = \gamma(M - \alpha)$.
Since $(w_i - \gamma c)^2 = (u_i - c)^2$ for any constant $c$,
\[
(w_i - \mathcal{Q}_\gamma(w_i))^2 - (w_i - R(w_i))^2 = (u_i - (M - \alpha))^2 - (u_i - M)^2 = 2\alpha(u_i - T).
\]

Splitting the sum at $C_\gamma$ and applying Steps 1--2,
\[
\|w - \mathcal{Q}_\gamma(w)\|_2^2 = \sum_{i=1}^d \bigl(w_i - R(w_i)\bigr)^2 + \sum_{i \in C_\gamma} \Bigl[\bigl(w_i - \mathcal{Q}_\gamma(w_i)\bigr)^2 - \bigl(w_i - R(w_i)\bigr)^2\Bigr] = E(w) + 2\alpha\sum_{i \in C_\gamma}(\gamma w_i - T).
\]
\end{proof}

\subsection{Proof of Theorem~\ref{thm:worst-case-bound}}

\thmworstcase*

\begin{proof}
Set $u_i := \gamma w_i$. Since $\gamma \in \{-1, +1\}$, multiplication by $\gamma$ preserves absolute error:
\[
|w_i - \mathcal{Q}_\gamma(w_i)| = |u_i - \mathcal{Q}_+(u_i)|,
\]
where $\mathcal{Q}_+$ denotes positive-scale symmetric quantization with step $\alpha$ on the alphabet $\mathcal{A}_q$.

If $i \notin C_\gamma$, then $|w_i - \mathcal{Q}_\gamma(w_i)| \le \alpha/2$: away from the boundary $u_i = T$, no clipping occurs and $\mathcal{Q}_+(u_i) = \alpha\,\lfloor u_i/\alpha \rceil$, while at the boundary the tied code may be clipped, but both candidates lie at distance $\alpha/2$.

If $i \in C_\gamma$, then Definition~\ref{def:clipped-set} gives $u_i \in (T, M]$, where $T = M - \alpha/2$ because $\alpha = M/2^{q-1}$.
Clipping maps $u_i$ to the largest representable value $M - \alpha$, so $|w_i - \mathcal{Q}_\gamma(w_i)| = u_i - (M - \alpha) \in (\alpha/2, \alpha]$.

Therefore
\[
\bigl(w_i - \mathcal{Q}_\gamma(w_i)\bigr)^2 \le \frac{\alpha^2}{4} + \frac{3\alpha^2}{4}\,\mathbf{1}_{\{i \in C_\gamma\}},
\]
and summing over $i$ yields $\|w - \mathcal{Q}_\gamma(w)\|_2^2 \le \tfrac{\alpha^2}{4}\bigl(d + 3|C_\gamma|\bigr)$.
Substituting $\alpha = M / 2^{q-1}$ gives $\tfrac{\alpha^2}{4} = 2^{-2q}\,M^2$, and writing $|C_\gamma| = \rho_\gamma d$ yields
\[
\|w - \mathcal{Q}_\gamma(w)\|_2^2 \;\le\; 2^{-2q}\,M^2\,d\,(1 + 3\rho_\gamma). \qedhere
\]
\end{proof}

\subsection{Proof of Corollary~\ref{cor:optimal-sign}}

\coroptimalsign*

\begin{proof}
By Theorem~\ref{thm:worst-case-bound}, the worst-case bound is strictly increasing in $\rho_\gamma = |C_\gamma|/d$, with a strictly positive prefactor $2^{-2q}\,M^2\,d$ that is independent of $\gamma$. Its minimizers over $\gamma \in \{-1,+1\}$ are therefore exactly $\argmin_{\gamma \in \{-1,+1\}} |C_\gamma|$.
Since $\gamma$ takes only the two values $\pm 1$, the condition $|C_{\gamma^\star}| \le |C_{-\gamma^\star}|$ implies that $\gamma^\star$ minimizes $|C_\gamma|$, and therefore minimizes the bound.
\end{proof}

\subsection{Proof of Corollary~\ref{cor:mse-optimal}}

\cormseoptimal*

\begin{proof}
The grids $\alpha\mathcal{A}_q$ and $-\alpha\mathcal{A}_q$ share every value between $-(2^{q-1}-1)\alpha$ and $(2^{q-1}-1)\alpha$.
For $|w_i| \le T$, the two squared quantization errors therefore agree, including at rounding ties, where both candidates lie at distance $\alpha/2$ from $w_i$.

For $w_i > T$, the reconstructions are $T - \alpha/2$ under $\gamma = +1$ and $T + \alpha/2$ under $\gamma = -1$, whether by rounding or by clipping. Hence
\[
\bigl(w_i - \mathcal{Q}_{+1}(w_i)\bigr)^2 - \bigl(w_i - \mathcal{Q}_{-1}(w_i)\bigr)^2 = \bigl(w_i - T + \tfrac{\alpha}{2}\bigr)^2 - \bigl(w_i - T - \tfrac{\alpha}{2}\bigr)^2 = 2\alpha(w_i - T).
\]
By reflection, each $w_i < -T$ contributes $-2\alpha(-w_i - T)$.
Summing over all coordinates gives
\[
\|w - \mathcal{Q}_{+1}(w)\|_2^2 - \|w - \mathcal{Q}_{-1}(w)\|_2^2 = 2\alpha D(w).
\]
Since $\alpha > 0$, the sign of this difference is the sign of $D(w)$, so $\gamma^\dagger = -\operatorname{sign}\bigl(D(w)\bigr)$ attains the smaller squared error whenever $D(w) \ne 0$. Both signs are optimal when $D(w) = 0$.
\end{proof}

\begin{remark}[$\ell_1$-Optimality of the Sign Rule]
\label{rem:l1-optimal-sign}
The $\ell_1$ error can be analyzed similarly to Corollary~\ref{cor:mse-optimal}. For any $\alpha > 0$, each coordinate with $w_i > T$ satisfies
\[
|w_i - \mathcal{Q}_{+1}(w_i)| - |w_i - \mathcal{Q}_{-1}(w_i)| = |w_i - T + \alpha/2| - |w_i - T - \alpha/2| = 2\min\!\{w_i - T, \alpha/2\}.
\]
By reflection, each coordinate with $w_i < -T$ contributes $-2\min\!\{-w_i - T, \alpha/2\}$.
For $|w_i| \le T$, the two absolute errors agree, including at rounding ties.

If $\alpha \ge M/2^{q-1}$, then $(|w_i| - T)_+ \le \alpha/2$ for every $i$, so
\[
\|w - \mathcal{Q}_{+1}(w)\|_1 - \|w - \mathcal{Q}_{-1}(w)\|_1
= 2\sum_{i=1}^{d}\operatorname{sign}(w_i)(|w_i| - T)_+ = 2D(w).
\]
Hence $\gamma^\dagger$ also minimizes the $\ell_1$ error when $\alpha \ge M/2^{q-1}$, with both signs optimal when $D(w) = 0$. Note that $\gamma^\dagger$ is not guaranteed to minimize the $\ell_1$ error for smaller $\alpha$. For example, with $q = 2$, $\alpha = 1$, and $w = (3, -2, -2)$, we have $T = 3/2$ and $D(w) = 1/2$. Thus $\gamma^\dagger = -1$ attains the smaller squared $\ell_2$ error ($3$ versus $4$) but the larger $\ell_1$ error ($3$ versus $2$).
\end{remark}

\subsection{Proof of Theorem~\ref{thm:signed-asymmetric}}

\thmsignedasymmetric*

\begin{proof}
For fixed \((s, z)\), the quantizer \(\mathcal{Q}(\cdot; s, z)\) maps into the grid \(s(\mathcal{A}_q - z)\).
Hence, the representable grid of \(\mathcal{Q}(\cdot; -\alpha, 0)\) is \(-\alpha\,\mathcal{A}_q\), while the representable grid of \(\mathcal{Q}(\cdot; \alpha, -1)\) is \(\alpha(\mathcal{A}_q + 1)\).
Since \(\mathcal{A}_q = \{-2^{q-1}, \ldots, 2^{q-1}-1\}\), these grids coincide:
\begin{align*}
-\alpha\,\mathcal{A}_q &\;=\; \alpha(\mathcal{A}_q + 1) \\
&\;=\; \alpha \cdot \{-(2^{q-1}-1), \ldots, 2^{q-1}\}.
\end{align*}
Therefore, \(\mathcal{Q}(\cdot; -\alpha, 0)\) and \(\mathcal{Q}(\cdot; \alpha, -1)\) have the same representable grid for any \(\alpha > 0\), independently of \(w\).
Moreover, each quantizer rounds \(w\) to a nearest element of this grid.
Hence, if \(w/\alpha \notin \mathbb{Z} + \tfrac{1}{2}\), that element is unique and the two quantizers agree pointwise, \textit{i.e.}, \(\mathcal{Q}(w; -\alpha, 0) = \mathcal{Q}(w; \alpha, -1)\).
\end{proof}

\section{Additional Experiments}
\label{sec:ablations}

\subsection{Expected acceptance rate under direct casting}
\label{sec:ablation-ear}

Table~\ref{tbl:ear_isolation} reports the expected acceptance rate for the direct-casting protocol of Table~\ref{tbl:sign_isolation}, over the same models, group size, and bit widths.
Both sign rules improve on the strictly positive convention in all 24 configurations.

\begin{table}[!t]
\centering
\caption{Expected acceptance rate (\%, $\uparrow$) under direct casting at group size 32, following the protocol of Table~\ref{tbl:sign_isolation}. Values are computed over the BF16 model's top $20$ tokens on WikiText2 and measured relative to BF16. Bold marks the best value in each column.}
\label{tbl:ear_isolation}
{\small
\begin{tabular}{cc ccccc cccc ccc}
\toprule
\multirow{2}{*}{$q$} & \multirow{2}{*}{$\gamma$} & \multicolumn{5}{c}{Qwen3} & \multicolumn{4}{c}{Qwen3.5} & \multicolumn{3}{c}{Llama3} \\
\cmidrule(lr){3-7} \cmidrule(lr){8-11} \cmidrule(lr){12-14}
 & & 0.6B & 1.7B & 4B & 8B & 14B & 0.8B & 2B & 4B & 9B & 1B & 3B & 8B \\
\midrule
\multirow{3}{*}{4} & $+1$ & 77.5 & 81.6 & 88.7 & 89.2 & 90.8 & 83.5 & 86.4 & 89.8 & 88.6 & 81.7 & 88.5 & 89.2 \\
 & $\gamma^\star$ & \textbf{79.0} & \textbf{84.8} & \textbf{89.7} & 90.6 & \textbf{91.9} & \textbf{84.4} & \textbf{87.8} & \textbf{90.4} & 89.5 & 83.6 & \textbf{90.0} & \textbf{91.0} \\
 & $\gamma^\dagger$ & \textbf{79.0} & \textbf{84.8} & \textbf{89.7} & \textbf{90.7} & \textbf{91.9} & \textbf{84.4} & \textbf{87.8} & \textbf{90.4} & \textbf{89.6} & \textbf{83.7} & \textbf{90.0} & \textbf{91.0} \\
\midrule
\multirow{3}{*}{3} & $+1$ & 57.7 & 61.3 & 75.6 & 76.6 & 78.3 & 63.4 & 64.0 & 75.3 & 76.3 & 53.9 & 65.0 & 70.8 \\
 & $\gamma^\star$ & 60.0 & \textbf{70.3} & 78.0 & \textbf{81.2} & \textbf{81.1} & 65.4 & 71.1 & \textbf{77.7} & 79.2 & 62.8 & 73.4 & \textbf{74.4} \\
 & $\gamma^\dagger$ & \textbf{60.3} & \textbf{70.3} & \textbf{78.1} & \textbf{81.2} & 80.8 & \textbf{65.5} & \textbf{71.4} & 77.2 & \textbf{79.3} & \textbf{63.4} & \textbf{73.6} & 74.1 \\
\bottomrule
\end{tabular}
}
\end{table}

\subsection{Sensitivity to group size}
\label{sec:ablation-groupsize}

Throughout our evaluation, we follow the \texttt{llama.cpp} block convention and use group size 32.
Here we vary the group size and measure its effect on the cardinality condition and sign rule of Corollary~\ref{cor:optimal-sign}.
We adopt the direct-casting protocol of Table~\ref{tbl:sign_isolation} at group sizes 16, 32, 64, and 128; the group size 32 entries reproduce the Qwen3 8B rows of that table.
Table~\ref{tbl:gs_coverage} reports coverage of the cardinality condition at each group size and bit width.
We observe that coverage decreases monotonically as the group size grows and the bit width decreases, with $q = 2$ dropping by $10.5$ percentage points from $92.0\%$ to $81.5\%$ and $q = 4$ dropping by only $1.7$ from $99.3\%$ to $97.6\%$.

\begin{table}[!t]
\centering
\caption{Cardinality condition coverage versus group size. Percentage of Qwen3 8B weight groups that satisfy the cardinality condition of Corollary~\ref{cor:optimal-sign}, at each group size and bit width. The group size 32 row corresponds to the measurement reported in Remark~\ref{rem:cardinality-margin}.}
\label{tbl:gs_coverage}
{\small
\begin{tabular}{cccc}
\toprule
Group size & $q=2$ & $q=3$ & $q=4$ \\
\midrule
16  & 92.0 & 97.6 & 99.3 \\
32  & 88.3 & 96.1 & 98.8 \\
64  & 84.8 & 94.4 & 98.3 \\
128 & 81.5 & 92.7 & 97.6 \\
\bottomrule
\end{tabular}
}
\end{table}

\begin{table}[!t]
\centering
\caption{Sign selection policies versus group size. WikiText2 perplexity ($\downarrow$) on Qwen3 8B under the direct-casting protocol of Table~\ref{tbl:sign_isolation}. We compare the strictly positive absmax grid ($\gamma = +1$) against the two closed-form sign rules ($\gamma^\star$ and $\gamma^\dagger$) and a random sign assignment. Random sign results report mean and standard deviation over five seeds.}
\label{tbl:gs_perplexity}
{\small
\begin{tabular}{cccccc}
\toprule
$q$ & Group size & $\gamma = +1$ & $\gamma^\star$ & $\gamma^\dagger$ & Random \\
\midrule
\multirow{4}{*}{3} & 16  & $2 \times 10^{9}$ & \textbf{9.59}  & 9.60  & 10.68 $\pm$ 0.23 \\
                   & 32  & 11.04 & 10.18 & \textbf{10.14} & 11.19 $\pm$ 0.22 \\
                   & 64  & 11.88 & \textbf{10.97} & 11.01 & 12.17 $\pm$ 0.18 \\
                   & 128 & 13.06 & \textbf{11.99} & 12.02 & 13.74 $\pm$ 0.23 \\
\midrule
\multirow{4}{*}{4} & 16  & 9.16 & \textbf{8.95} & \textbf{8.95} & 9.20 $\pm$ 0.09 \\
                   & 32  & 9.06 & \textbf{8.91} & 8.92 & 9.12 $\pm$ 0.05 \\
                   & 64  & 8.89 & 8.79 & \textbf{8.78} & 9.00 $\pm$ 0.07 \\
                   & 128 & 9.12 & \textbf{8.99} & \textbf{8.99} & 9.27 $\pm$ 0.03 \\
\bottomrule
\end{tabular}
}
\end{table}

Table~\ref{tbl:gs_perplexity} reports perplexity under four sign policies. Both closed-form rules improve over the strictly positive absmax grid and over the random baseline at every group size and bit width, and they remain within $0.04$ perplexity of each other throughout.
One result stands out:
at $q = 3$ and group size 16, the strictly positive absmax grid gives a perplexity of $2 \times 10^{9}$, while both closed-form rules and the random baseline remain stable.
A finer partition promotes many previously interior weights to group maxima, so more groups have a dominant positive outlier that $\gamma = +1$ clips.
Sub-optimal polarity decisions alone do not explain the failure: a random sign also disagrees with $\gamma^\star$ on half of all groups in expectation, yet remains stable.
Unlike a random sign, $\gamma = +1$ always clips the positive tail, so every clipping error $w_i - \mathcal{Q}_{+1}(w_i)$ is positive.
The experiments therefore suggest that these errors accumulate.
A controlled diagnostic is consistent with this: holding the scale and the rounding fixed while extending only the upper endpoint from $2^{q-1}-1$ to $2^{q-1}$ eliminates the failure.
Thus, although group size affects both cardinality coverage and absolute perplexity, the benefit of selecting $\gamma^\star$ persists across all tested granularities.

\subsection{Composition with learned rotations and grid parameters}
\label{sec:ablation-learned-llama}

We next ask whether the same degree of freedom remains useful when the quantization pipeline is allowed to adapt from data. To this end, we jointly learn rotations and grid parameters on FineWeb~\cite{penedo2024fineweb} at the same insertion points as the fixed pipeline (Figure~\ref{fig:rotation-graph}), using Table~\ref{tbl:llama8b_full} as an initialization, then freeze them before applying Qronos~\cite{zhang2026qronos} with dynamic \texttt{Q8\_0} activation quantization. Rotations are optimized on the Stiefel manifold using Cayley SGD as in SpinQuant~\cite{liu2025spinquant}. For the grid parameters, \texttt{Q4\_1} learns the scale and zero-point offset jointly as in LSQ+~\cite{bhalgat2020lsq+}, while \texttt{Q4\_0} follows LSQ~\cite{esser2019learned} with the scale constrained to be strictly positive ($s > 0$) under $\gamma = +1$ and non-zero ($s \neq 0$) under $\gamma^\star$. We use analogous parameterizations for \texttt{Q3} and \texttt{Q2}, with full optimization details deferred to Appendix~\ref{sec:experiment-details}. Appendix~\ref{sec:ablation-grid} studies this same pipeline on Qwen3 8B, with an added ablation on how the grid is obtained during gradient-based PTQ (Table~\ref{tbl:qwen8b_grids}).

Table~\ref{tbl:llama8b_learned} shows that the signed symmetric grid continues to improve over the conventional symmetric grid after joint rotation and grid learning. At W4, the two sign choices have similar recovery: \texttt{Q4\_0} under $\gamma^\star$ changes WikiText2 perplexity from $6.7$ to $6.6$ and recovery from $99.3\%$ to $99.5\%$. We report small differences between recovery point estimates without claiming statistical significance. The effect is larger at lower precision: at W3, $\gamma^\star$ improves perplexity from $7.3$ to $7.1$ and recovery from $92.0\%$ to $96.4\%$; at W2, it reduces perplexity from $103.8$ to $12.9$ and improves recovery from $50.5\%$ to $68.8\%$. The comparison to the asymmetric format is mixed but informative: \texttt{Q3\_1} retains superior recovery at W3, while recovery is similar at W4 and $\gamma^\star$ slightly exceeds \texttt{Q2\_1} at W2 ($68.8\%$ versus $67.0\%$) despite worse perplexity ($12.9$ versus $12.2$). This perplexity-versus-recovery inversion is not driven by a single outlier, as $\gamma^\star$ improves on three of the five tasks and ties a fourth, and is consistent with prior observations that the two metrics can decouple at low precision~\cite{sanjeet2026mixquant, zhang2026qronos}.
The intuition is analogous to overfitting: Zhang et al.~\yrcite{zhang2025provable} (Remark 3.8) decompose generalization error into a calibration reconstruction term and a parameter-proximity term, which can be at odds.

\begin{table}[!t]
\centering
\caption{Composition with jointly learned rotations and grid parameters. WikiText2 perplexity ($\downarrow$), zero-shot accuracy (\%, $\uparrow$), and recovery ($\uparrow$) on Llama3 8B for few-bit GGUF formats. We jointly learn rotations and grid parameters on FineWeb using Table~\ref{tbl:llama8b_full} as an initialization, then freeze them before applying Qronos with dynamic \texttt{Q8\_0} activation quantization. Recovery is the geometric mean of zero-shot accuracy relative to BF16 over the five reported tasks. The BF16 reference was evaluated separately from that of Table~\ref{tbl:llama8b_full}, so values should be compared within this table only.}
\label{tbl:llama8b_learned}
{\small
\begin{tabular}{ll ccccccc}
\toprule
Format & $\gamma$ & Wiki2 $\downarrow$ & ARC-C & ARC-E & HellaS & WinoG & PIQA & Recovery \% \\
\midrule
BF16 & --- & 6.5 & 50.1 & 77.4 & 57.4 & 65.0 & 80.4 & --- \\
\midrule
\texttt{Q4\_0} & $+1$ & 6.7 & 48.8 & \textbf{77.1} & 57.1 & \textbf{65.6} & 79.5 & 99.3 \\
\texttt{Q4\_0} & $\gamma^\star$ & \textbf{6.6} & \textbf{49.2} & 77.0 & 57.1 & \textbf{65.6} & \textbf{80.0} & \textbf{99.5} \\
\texttt{Q4\_1} & --- & \textbf{6.6} & 48.6 & 76.9 & \textbf{57.6} & 64.4 & \textbf{80.0} & 99.1 \\
\midrule
\texttt{Q3\_0} & $+1$ & 7.3 & 40.2 & 71.1 & 54.8 & 62.8 & 78.1 & 92.0 \\
\texttt{Q3\_0} & $\gamma^\star$ & \textbf{7.1} & 45.0 & 74.9 & \textbf{55.9} & \textbf{64.7} & \textbf{79.4} & 96.4 \\
\texttt{Q3\_1} & --- & \textbf{7.1} & \textbf{47.3} & \textbf{76.3} & 55.7 & 63.9 & 78.9 & \textbf{97.3} \\
\midrule
\texttt{Q2\_0} & $+1$ & 103.8 & 21.4 & 25.6 & 25.9 & \textbf{51.1} & 52.8 & 50.5 \\
\texttt{Q2\_0} & $\gamma^\star$ & 12.9 & \textbf{27.1} & \textbf{51.0} & 38.4 & 50.6 & \textbf{66.9} & \textbf{68.8} \\
\texttt{Q2\_1} & --- & \textbf{12.2} & 24.7 & 48.7 & \textbf{39.4} & 50.6 & 65.5 & 67.0 \\
\bottomrule
\end{tabular}
}
\end{table}

\subsection{Recomputing the grid during optimization}
\label{sec:ablation-grid}

This section ablates how the grid is obtained, and extends the case study of Table~\ref{tbl:llama8b_learned} to Qwen3 8B.
Tables~\ref{tbl:llama8b_full} and~\ref{tbl:llama8b_learned} construct the weight grid once, before optimization, and either hold it fixed or learn from that initialization.
Learned rotations transform the weight tensors at every gradient step, so a grid set at initialization describes weights that have moved.
A third option is to recompute the grid during optimization, mirroring the way activation grids are computed.

We compare these three ways of obtaining the grid within the learned-rotation pipeline, on Qwen3 8B at 2 bits: (1) \emph{learned} parameters are initialized and optimized as in Table~\ref{tbl:llama8b_learned}; (2) \emph{fixed} parameters are held at their initialization; and (3) \emph{dynamic} parameters are recomputed from the transformed weights on every forward pass.

\begin{table}[!t]
\centering
\caption{Recomputing the grid during optimization. WikiText2 perplexity ($\downarrow$), zero-shot accuracy (\%, $\uparrow$), and average recovery ($\uparrow$) on Qwen3 8B at 2 bits, under the learned-rotation pipeline of Table~\ref{tbl:llama8b_learned}. \emph{Learned} grid parameters are initialized from the corresponding closed-form grid and then optimized; \emph{fixed} parameters are held at their initialization; \emph{dynamic} parameters are recomputed from the transformed weights on every forward pass. Recovery is the geometric mean of zero-shot accuracy relative to BF16 over the five reported tasks. The BF16 reference shown was evaluated separately from that of Table~\ref{tbl:qwen8b_main}, so values should be compared within this table only.}
\label{tbl:qwen8b_grids}
{\small
\begin{tabular}{lll cccccc c}
\toprule
Grid & Format & $\gamma$ & Wiki2 $\downarrow$ & ARC-C & ARC-E & HellaS & WinoG & PIQA & Recovery \% \\
\midrule
--- & BF16 & --- & 8.68 & 55.2 & 83.2 & 57.0 & 64.4 & 76.6 & --- \\
\midrule
\multirow{3}{*}{learned} & \texttt{Q2\_0} & $+1$ & 13.00 & 28.2 & 57.2 & 38.2 & 57.5 & 64.5 & 70.7 \\
 & \texttt{Q2\_0} & $\gamma^\star$ & 12.03 & 31.8 & 57.7 & 41.6 & 58.5 & 69.0 & 75.1 \\
 & \texttt{Q2\_1} & --- & 11.88 & 30.4 & 58.9 & 40.5 & 54.9 & 66.1 & 72.8 \\
\midrule
\multirow{3}{*}{fixed} & \texttt{Q2\_0} & $+1$ & 13.09 & 29.9 & 60.0 & 37.8 & 57.1 & 66.5 & 72.4 \\
 & \texttt{Q2\_0} & $\gamma^\star$ & 11.98 & 33.1 & 62.2 & 42.7 & 56.7 & 68.3 & 76.6 \\
 & \texttt{Q2\_1} & --- & 12.00 & 30.0 & 58.3 & 40.1 & 56.4 & 67.2 & 72.9 \\
\midrule
\multirow{3}{*}{dynamic} & \texttt{Q2\_0} & $+1$ & 12.88 & 27.0 & 53.6 & 38.4 & 57.1 & 65.5 & 69.4 \\
 & \texttt{Q2\_0} & $\gamma^\star$ & \textbf{11.39} & \textbf{35.8} & \textbf{66.1} & \textbf{42.9} & \textbf{59.0} & \textbf{69.3} & \textbf{79.7} \\
 & \texttt{Q2\_1} & --- & 11.89 & 29.7 & 57.1 & 41.7 & 54.0 & 67.7 & 72.5 \\
\bottomrule
\end{tabular}
}
\end{table}

On Qwen3 8B, we observe that dynamic construction is the strongest of the three for the signed format.
\texttt{Q2\_0} under $\gamma^\star$ leads every reported column in this regime, reaching $11.39$ perplexity and $79.7\%$ recovery, against $76.6\%$ when fixed and $75.1\%$ when learned.
Note that dynamic and learned construction impose a similar end-to-end training overhead, each roughly $10\%$ above fixed.

These results reinforce two points.
First, the strongest configuration is precisely the one our analysis describes (Theorem~\ref{thm:worst-case-bound}).
Dynamic construction for \texttt{Q2\_0} under $\gamma^\star$ recomputes the signed absmax grid of Section~\ref{sec:signed_sym} at each forward pass and yields the best accuracy for Qwen3 8B by a clear margin; learned grids depart from this grid after initialization.
This suggests that our analysis remains predictive, and the signed grid competitive, even in gradient-based PTQ pipelines.
Second, the signed absmax rule is cheap enough to recompute at every forward pass (Remark~\ref{rem:oracle}) and composes with the straight-through estimator~\cite{bengio2013estimating} used in this pipeline.

\subsection{Composition with alternative PTQ methods}
\label{sec:ablation-smoothquant}

\begin{table}[!t]
\centering
\caption{Composition with SmoothQuant. WikiText2 perplexity ($\downarrow$), zero-shot accuracy (\%, $\uparrow$), and average recovery ($\uparrow$) on Qwen3 8B at 3 bits, using channel rescaling in place of the rotations and error correction of Table~\ref{tbl:llama8b_full}. Recovery is the geometric mean of zero-shot accuracy relative to BF16 over the five reported tasks. The BF16 reference shown was evaluated separately from that of Table~\ref{tbl:qwen8b_main}, so values should be compared within this table only.}
\label{tbl:qwen8b_smoothquant}
{\small
\begin{tabular}{ll cccccc c}
\toprule
Format & $\gamma$ & Wiki2 $\downarrow$ & ARC-C & ARC-E & HellaS & WinoG & PIQA & Recovery \% \\
\midrule
BF16 & --- & 8.68 & 55.2 & 83.2 & 57.0 & 64.4 & 76.6 & --- \\
\midrule
\texttt{Q3\_0} & $+1$ & 11.32 & 46.2 & 73.9 & 52.0 & 59.7 & 75.2 & 90.8 \\
\texttt{Q3\_0} & $\gamma^\star$ & 10.33 & \textbf{48.5} & \textbf{78.7} & \textbf{53.3} & \textbf{61.2} & \textbf{75.7} & \textbf{93.9} \\
\texttt{Q3\_1} & --- & \textbf{10.18} & 44.8 & 76.3 & 51.5 & 59.8 & 74.0 & 90.4 \\
\bottomrule
\end{tabular}
}
\end{table}

In principle, the signed absmax grid is independent of the surrounding PTQ pipeline. Here we compose it with SmoothQuant~\cite{xiao2023smoothquant}, which rescales channels to migrate quantization difficulty from activations to weights and requires neither rotations nor error correction. Applying it to Qwen3 8B at 3 bits, Table~\ref{tbl:qwen8b_smoothquant} shows that \texttt{Q3\_0} under $\gamma^\star$ gives the best accuracy on each of the five tasks, and therefore the best recovery, $93.9\%$ against $90.8\%$ under $\gamma = +1$ and $90.4\%$ for \texttt{Q3\_1}. The asymmetric \texttt{Q3\_1} gives the lowest perplexity, $10.18$ against $10.33$, but the lowest recovery, a further instance of the decoupling discussed in Section~\ref{sec:results}.

\section{Implementation Details for \texttt{llama.cpp}}
\label{sec:implementation}

This appendix expands on the inference measurements in Table~\ref{tbl:cost}. The goal is to quantify the systems cost of the deployed quantization format, independent of the sign selection rules (Section~\ref{sec:signed_sym}). In particular, the measurements compare symmetric and asymmetric GGUF formats with matched group size and bit width. Since signed symmetric quantization uses the same deployed format as the standard symmetric quantization, these measurements apply to both; the sign rules change scale selection offline, not the runtime kernel path.

We consider two standard 4-bit GGUF formats: \texttt{Q4\_1} (asymmetric) and \texttt{Q4\_0} (symmetric).
No comparable 2-bit formats existed in \texttt{llama.cpp} when this study began, so we implemented \texttt{Q2\_1} (asymmetric) and \texttt{Q2\_0} (symmetric) to enable a direct comparison of signed symmetric and asymmetric quantization at INT2 precision.
A related symmetric 2-bit format (\texttt{Q2\_0}, with group size 64) was contributed upstream independently while this work was in progress and has since been merged~\cite{khosravi2026q2}.
The measurements reported here predate that merge and use our custom implementation.
We benchmark CPU inference throughput using the \texttt{llama.cpp} framework~\cite{llama_cpp} on an AMD EPYC\texttrademark{} ``Turin'' 128-core processor with the Zen Deep Neural Network backend, ZenDNN(L) v1.0.0.

\subsection{Formats and block layout}
\label{sec:impl-formats}

Consistent with the \texttt{Q4\_0} and \texttt{Q4\_1} standard, all formats use a block size of 32 weights.
For \texttt{Q2}, each block packs 32 two-bit indices into 8~bytes, with four indices per byte.
For \texttt{Q4}, each block packs 32 four-bit indices into 16~bytes, with two indices per byte.

\begin{center}
{\small
\begin{tabular}{lll}
\toprule
Format & Fields & Size \\
\midrule
\texttt{Q2\_0} (symmetric)  & \texttt{s} (FP16) + \texttt{qs[8]}                      & 10 bytes \\
\texttt{Q2\_1} (asymmetric) & \texttt{s} (FP16) + \texttt{m} (FP16) + \texttt{qs[8]}  & 12 bytes \\
\texttt{Q4\_0} (symmetric)  & \texttt{s} (FP16) + \texttt{qs[16]}                     & 18 bytes \\
\texttt{Q4\_1} (asymmetric) & \texttt{s} (FP16) + \texttt{m} (FP16) + \texttt{qs[16]} & 20 bytes \\
\bottomrule
\end{tabular}
}
\end{center}

\noindent
\texttt{Q4\_0} and \texttt{Q2\_0} store a single scale $s = \gamma\alpha$; the format supports either sign without additional metadata.
\texttt{Q4\_1} and \texttt{Q2\_1} also store a scale $s$ and minimum $m$.

\subsection{Measurement setup}
\label{sec:impl-setup}

\paragraph{Hardware.}
CPU experiments run on an AMD EPYC\texttrademark{} 9755 (Zen~5, ``Turin'') system with
128 physical cores per socket (256 logical), two NUMA nodes, and 1.5~TB DDR5.
The CPU supports AVX-512 (F/BW/DQ/VL), AVX-512VNNI, AVXVNNIINT8, AVX2, and FMA.

\paragraph{Software and builds.}
All experiments use \texttt{llama.cpp} at commit \texttt{b1be68e8} (master, ggml-org/llama.cpp),
compiled with GCC~14.3.0 and CMake~4.2.1 on Ubuntu~22.04 (Linux~5.15) in \texttt{Release}
mode with \texttt{GGML\_NATIVE=ON} and \texttt{GGML\_OPENMP=ON}, enabling
\texttt{-march=native} and all ISA extensions supported by the target CPU.
We evaluate two configurations from this same commit:
(i) an \emph{unmodified upstream} baseline build, supporting BF16, \texttt{Q4\_0}, and \texttt{Q4\_1} with
no custom kernels or code modifications, and
(ii) a \emph{modified build} that adds \texttt{Q2\_0}/\texttt{Q2\_1} quantization support and AVX-512--optimized
dot product kernels for all quantized types.
In the baseline build, \texttt{Q4\_0}/\texttt{Q4\_1} dot product kernels do not have explicit AVX-512 paths and execute via 256-bit AVX2 VNNI instructions, while other operations (\textit{e.g.}, attention, RMSNorm, SiLU, softmax) use existing 512-bit AVX-512 kernels.
Models are converted from HuggingFace safetensors to BF16 GGUF and then quantized to \texttt{Q4\_0}, \texttt{Q4\_1}, \texttt{Q2\_0}, and \texttt{Q2\_1}. However, we note that one could similarly export Brevitas-quantized models that are compatible with the \texttt{llama.cpp} runtime.

\paragraph{Benchmarking and metrics.}
Throughput is measured using \texttt{llama-bench}.
We evaluate a 512-token prefill followed by 512-token autoregressive decode,
with prefill batch size 512 and 128 threads.
All CPU runs are pinned to a single NUMA node using
\texttt{numactl --cpunodebind=0 --membind=0},
mapping one thread per physical core and avoiding cross-NUMA traffic.
We report tokens/s for: (i) \emph{prefill (pp512)}, a compute-bound phase dominated by matrix
multiplications, and (ii) \emph{decode (tg512)}, a memory-bandwidth-bound phase.
Each experiment performs one warmup iteration followed by 20 timed repetitions; reported results are the mean over repetitions. Across all reported configurations, over 90\% exhibit a coefficient of variation (standard deviation divided by mean) below 6\%.
We note that CPU and AMD Instinct\texttrademark{} MI300X experiments were conducted as of February 2026, and AMD Radeon\texttrademark{} AI Pro R9700 experiments as of September 2026.
Absolute throughput depends on system configuration and workload, so we report format-to-format ratios measured on a single system and detail our measurement protocol.

\subsection{Kernel-level cost gap}
\label{sec:impl-mechanism}

The throughput comparison is between deployed formats, not between sign selection policies.
All three sign rules ($\gamma = +1$, $\gamma^\star$, and $\gamma^\dagger$) deploy as a symmetric format with $z = 0$; the only difference is the sign assigned to each per-group scale.
When scales are stored as signed real values, the sign introduces no additional metadata.
\texttt{Q4} measurements use the unmodified upstream build; \texttt{Q2} measurements use the modified build described in Appendix~\ref{sec:impl-setup}, which applies the same kernels to both \texttt{Q2} formats.
Therefore, Table~\ref{tbl:cost} characterizes the runtime profile inherited by signed symmetric quantization from the symmetric family.

\paragraph{Unpacking overhead.}
After loading 8~bytes of packed 2-bit indices from a \texttt{Q2} format, the kernel must expand them to 32 individual bytes in element order.
This requires broadcasting the 8-byte value to a 256-bit register, extracting four shift-and groups, and interleaving the results with two \texttt{unpacklo\_epi8} and one \texttt{unpacklo\_epi16}, approximately 10~SIMD instructions per block.
By comparison, the 4-bit nibble unpack (\texttt{bytes\_from\_nibbles\_32}) needs only a 128-bit load, one shift, and one mask: approximately 3~instructions.
This $3\times$ difference in unpack cost is a fixed overhead shared by both \texttt{Q2\_0} and \texttt{Q2\_1}, and it is the primary reason the symmetric-over-asymmetric throughput advantage is smaller at 2~bits than at 4~bits (Table~\ref{tbl:cost}).
On the AVX-512 path, a dual-block unpack primitive loads two blocks' \texttt{qs} fields into a single 128-bit register and produces a 512-bit result (64~bytes), halving the number of unpack calls per iteration.

\paragraph{Zero-point cost in quantized GEMM.}
The asymmetric format also adds work in the GEMM itself~\cite{jacob2018quantization, krishnamoorthi2018quantizing}. Consider one output element $Y_{ik} = \sum_j X_{ij} W_{kj}$ with $W$ quantized in groups of size $g$, so that within group $b$, $W_{kj} \approx s_{kb}(W_{q,kj} - z_{kb})$. Expanding,
\[
Y_{ik} \;\approx\; \sum_b s_{kb} \underbrace{\sum_{j \in b} X_{ij}\, W_{q,kj}}_{\text{integer GEMM}} \;-\; \sum_b s_{kb} z_{kb} \underbrace{\sum_{j \in b} X_{ij}}_{\text{row-sum correction}}.
\]
Both formats compute the integer GEMM term. The asymmetric format ($z \ne 0$) additionally accumulates a per-block row-sum of $X$ scaled by $sz$ per output, while the symmetric format ($z = 0$) skips this term entirely.

\subsection{Additional runtime experiments}
\label{sec:impl-runtime}

\begin{table}[!t]
\centering
\caption{GPU inference throughput of \texttt{Q4\_0} (symmetric) relative to \texttt{Q4\_1} (asymmetric) on an AMD Instinct\texttrademark{} MI300X with the ROCm\texttrademark{} 7.0 backend. Each cell reports the symmetric-to-asymmetric ratio of prefill and decode throughput at \texttt{Q4} precision. Values $>1$ favor the symmetric format. Memory ratios are device-independent and match Table~\ref{tbl:cost}.}
\label{tbl:cost_gpu}
{\small
\begin{tabular}{c ccccc cccc ccc}
\toprule
\multirow{2}{*}{Metric} & \multicolumn{5}{c}{Qwen3} & \multicolumn{4}{c}{Qwen3.5} & \multicolumn{3}{c}{Llama~3} \\
\cmidrule(lr){2-6} \cmidrule(lr){7-10} \cmidrule(lr){11-13}
 & 0.6B & 1.7B & 4B & 8B & 14B & 0.8B & 2B & 4B & 9B & 1B & 3B & 8B \\
\midrule
prefill ($\uparrow$) & 1.07 & 1.06 & 1.07 & 1.14 & 1.14 & 1.01 & 1.01 & 1.01 & 1.03 & 1.01 & 1.01 & 1.04 \\
decode ($\uparrow$)  & 0.99 & 1.00 & 0.98 & 0.98 & 0.99 & 0.99 & 0.99 & 0.99 & 0.98 & 1.00 & 0.99 & 0.98 \\
\bottomrule
\end{tabular}
}
\end{table}

\paragraph{GPU backends.}
We apply the measurement protocol of Table~\ref{tbl:cost} to a GPU backend, using an AMD Instinct\texttrademark{} MI300X with ROCm\texttrademark{} 7.0. Table~\ref{tbl:cost_gpu} shows that the symmetric format delivers $1.01$--$1.14\times$ the prefill throughput of the asymmetric format, while the decode differences fall within run-to-run variation. The memory ratios are device-independent and match Table~\ref{tbl:cost}. The decode advantage therefore depends on the platform.

\begin{table}[!t]
\centering
\caption{Inference throughput of \texttt{Q4\_0} (uniform) relative to \texttt{IQ4\_NL} (non-uniform) on the AMD EPYC\texttrademark{} ``Turin'' CPU of Table~\ref{tbl:cost}. Each cell reports the uniform-to-non-uniform ratio of prefill and decode throughput. Both formats use 4.5 bits per weight, so this comparison isolates throughput at matched storage. Values $>1$ favor the uniform format.}
\label{tbl:cost_iq4nl}
{\small
\begin{tabular}{c ccccc}
\toprule
\multirow{2}{*}{Metric} & \multicolumn{5}{c}{Qwen3} \\
\cmidrule(lr){2-6}
 & 0.6B & 1.7B & 4B & 8B & 14B \\
\midrule
prefill ($\uparrow$) & 1.04 & 1.06 & 1.08 & 1.04 & 1.08 \\
decode ($\uparrow$)  & 1.03 & 1.04 & 1.07 & 1.06 & 0.95 \\
\bottomrule
\end{tabular}
}
\end{table}

\paragraph{Non-uniform formats.}
We compare \texttt{Q4\_0} against \texttt{IQ4\_NL}, a non-uniform GGUF format implemented in \texttt{llama.cpp}, using the same setup from Table~\ref{tbl:cost}. Both formats use 4.5 bits per weight. Table~\ref{tbl:cost_iq4nl} shows that \texttt{Q4\_0} delivers $1.04$--$1.08\times$ the prefill throughput and $0.95$--$1.07\times$ the decode throughput of \texttt{IQ4\_NL}. Again, \texttt{Q4\_0} uses the same execution path under either sign selection rule.

\section{Experimental Setup}
\label{sec:experiment-details}

This appendix provides implementation details for the experiments in Section~\ref{sec:results}. We separate the quantization graph from the quantization pipeline. The graph specifies where rotations and quantizers are inserted into the model, while the pipeline specifies how grid parameters, rotations, and rounding decisions are chosen.

Except where noted (Appendix~\ref{sec:ablation-groupsize} varies the group size), we quantize linear projection weights using group size 32, matching the GGUF block convention used by \texttt{llama.cpp}. Except for the direct-casting experiments (Tables~\ref{tbl:sign_isolation} and~\ref{tbl:ear_isolation}--\ref{tbl:gs_perplexity}), we evaluate models with dynamic \texttt{Q8\_0} activation quantization. Weight grids are calculated (or learned) per group, while activation grids are computed dynamically from the corresponding (potentially transformed) activations per-token during each forward pass.

\subsection{Direct casting}
\label{sec:exp-direct}

The direct-casting experiments in Table~\ref{tbl:sign_isolation} isolate the effect of sign selection. We quantize all linear projection weights groupwise with group size 32 and do not use rotations, error correction, learned grid parameters, calibration data, or activation quantization. The conventional absmax baseline uses the signed integer alphabet with zero point $z = 0$, strictly positive scale $s = \alpha > 0$, and $\alpha = \|w\|_\infty / 2^{q-1}$ per group. The signed symmetric grids use the same alphabet and zero point, but set $s = \gamma\alpha$, where $\gamma \in \{+1, -1\}$ is chosen by Equation~\ref{eq:sign_rule} for the signed absmax rule and by Equation~\ref{eq:mse_sign_rule} for the MSE-optimal sign rule. All values are rounded to nearest and clipped to the target integer alphabet.

\subsection{Algorithm details}
\label{sec:exp-algo}

\begin{figure}[!t]
\centering
\includegraphics[width=\textwidth]{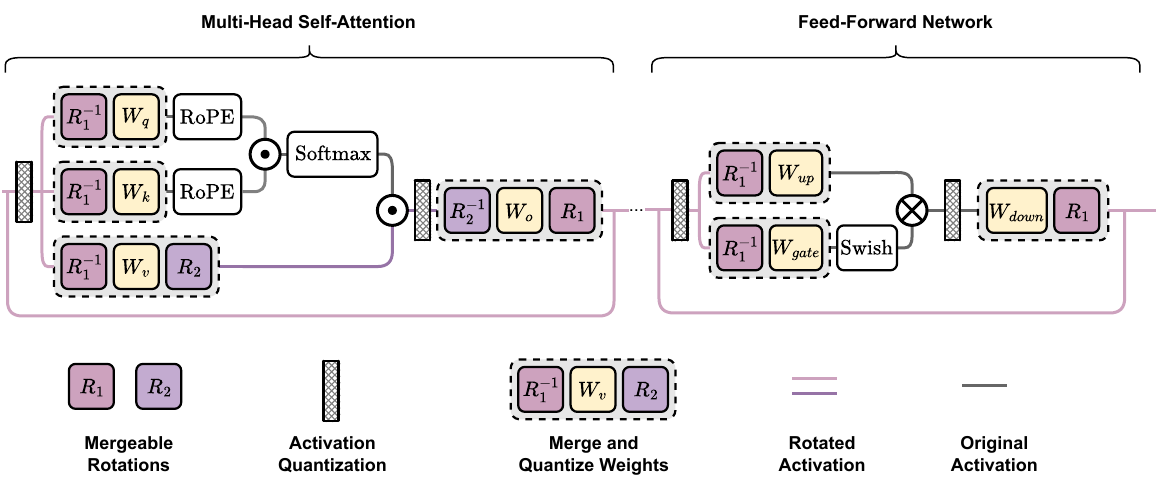}
\caption{We illustrate the mergeable rotation graph used by the PTQ pipelines behind Tables~\ref{tbl:llama8b_full} and~\ref{tbl:llama8b_learned}, with full pipeline details in Appendices~\ref{sec:exp-algo} and~\ref{sec:exp-learned}. Normalized Hadamard rotations are inserted only at rotation-invariant points where they fold into adjacent linear projections at deployment, leaving the graph architecture, and therefore the \texttt{llama.cpp} runtime contract, unchanged. This figure is adapted from Liu et al.~\yrcite{liu2025spinquant}.}
\label{fig:rotation-graph}
\end{figure}

The experiments in Table~\ref{tbl:llama8b_full} evaluate whether the signed grid remains useful when composed with a practical PTQ pipeline. We first insert normalized Hadamard rotations into rotation-invariant regions of the model only where rotations can be fully merged, as shown in Figure~\ref{fig:rotation-graph}. Note that the \texttt{llama.cpp} runtime assumes a fixed contract on the model (\textit{i.e.}, no changes to the graph architecture and therefore no online rotations can be added), which is a notable departure from QuaRot~\cite{ashkboos2024quarot} and SpinQuant~\cite{liu2025spinquant}. These rotations are merged before grid calculation, so weight grids are calculated based on the rotated weights. Activations are quantized dynamically with \texttt{Q8\_0}, with activation quantization grids computed from the rotated activations. In \texttt{llama.cpp}~\cite{llama_cpp}, \texttt{Q8\_0} on activations assumes a ``narrow range'' quantizer, where activations are scaled to be symmetric such that $\mathcal{A}_q = \{-K, \ldots, K\}$ for $K = 2^{q-1} - 1$. We emulate this quantizer within Brevitas during calibration.

After the rotations and grids are fixed, we apply Qronos as the adaptive rounding method~\cite{zhang2026qronos}. Qronos is run with the same transformed graph and dynamic activation-quantized execution path used for evaluation, so that the rounding decisions account for the mismatch between the quantized model and its full-precision counterpart under \texttt{Q8\_0} activation quantization.

For the studied formats, \texttt{Q4\_0} under $\gamma = +1$ uses the conventional symmetric grid with $z = 0$ and $s > 0$. \texttt{Q4\_0} under $\gamma^\star$ and $\gamma^\dagger$ uses the same signed alphabet and $z = 0$, but sets $s = \gamma\alpha$ with $\gamma$ chosen by Equation~\ref{eq:sign_rule} or Equation~\ref{eq:mse_sign_rule}, respectively. \texttt{Q4\_1} uses the corresponding asymmetric min-max grid. We use analogous definitions for \texttt{Q3} and \texttt{Q2}.

For Qronos, we use 128 calibration sequences of length 2048. The covariance matrix (also referred to as the Hessian) is damped as in Qronos, with damping parameter $\lambda = 10^{-4}\,\sigma_1$ where $\sigma_1$ is the largest singular value of the covariance matrix. We quantize weights in descending order of the diagonal of the covariance matrix~\cite{gptq}. All Qronos hyperparameters are held fixed across \texttt{Q4\_1} and all three sign choices of \texttt{Q4\_0} at the same bit width.

\subsection{Hyperparameters for jointly learned rotations and grids}
\label{sec:exp-learned}

Table~\ref{tbl:llama8b_learned} evaluates a data-driven pipeline in which rotations and grid parameters are learned jointly before applying Qronos. The goal is to test whether the signed scale remains useful when the quantization pipeline is allowed to adapt from data, rather than relying only on the closed-form signed absmax rule.

We initialize rotations and grid parameters from the fixed-rotation pipeline in Table~\ref{tbl:llama8b_full}. Starting from this initialization, we jointly optimize the rotations and grid parameters for 100 steps on FineWeb~\cite{penedo2024fineweb}, using batches of 8 sequences. Rotations are optimized on the Stiefel manifold using Cayley SGD, following SpinQuant~\cite{liu2025spinquant}. Quantized weights and activations are simulated during optimization using the straight-through estimator~\cite{bengio2013estimating, courbariaux2015binaryconnect}.

The three grid families differ only in their parameter constraints. For \texttt{Q4\_1}, we learn the scale and zero-point offset jointly, as in LSQ+~\cite{bhalgat2020lsq+}. We follow LSQ~\cite{esser2019learned} when learning scales for \texttt{Q4\_0}, constraining the scale to be strictly positive ($s > 0$) under $\gamma = +1$ and non-zero ($s \neq 0$) under $\gamma^\star$, with the latter initialized from the signed absmax grid. We use analogous parameterizations for \texttt{Q3} and \texttt{Q2}. After the 100-step joint optimization, the learned rotations are merged and the grid parameters are frozen, and Qronos is applied using the same adaptive rounding setup described in Appendix~\ref{sec:exp-algo}.

\subsection{Implementation notes}
\label{sec:exp-impl}

All quantization experiments are implemented using Brevitas~\cite{brevitas} and run on an AMD Instinct\texttrademark{} MI325X GPU with 256~GB of GPU memory. The learned grid experiments use the same quantization graph as the fixed-rotation experiments; only the procedure for choosing rotations and grid parameters changes. In particular, \texttt{Q4\_0} shares the same deployed symmetric inference path in \texttt{llama.cpp} under either sign choice. The signed variant changes the sign of the stored scale but does not introduce a zero point or any additional runtime arithmetic.

LightEval~\cite{lighteval} reports bootstrap standard errors alongside each downstream accuracy; across our evaluations these are approximately one percentage point, the largest being ARC-Challenge at $1.46$ points. Apart from the random sign baseline of Table~\ref{tbl:gs_perplexity}, our quantization procedure is deterministic for a fixed calibration set. We do not estimate calibration set standard error, which would require repeating each pipeline over multiple calibration draws.

\end{document}